\documentclass{article}

\usepackage{microtype}
\usepackage{times}
\usepackage{graphicx} 
\usepackage{amsmath,amssymb,amsfonts,amsthm}
\usepackage{xcolor}
\usepackage{multirow}
\usepackage{booktabs}
\usepackage{subcaption}
 \DeclareMathOperator*{\argmin}{argmin}
 
\usepackage{booktabs} 

\usepackage{hyperref}

\usepackage[accepted]{icml2024} 

\usepackage{amsmath}
\usepackage{amssymb}
\usepackage{mathtools}
\usepackage{amsthm}

\usepackage[capitalize,noabbrev]{cleveref}

\theoremstyle{plain}
\newtheorem{theorem}{Theorem}[section]

\newtheorem{lemma}[theorem]{Lemma}
\newtheorem{observation}[theorem]{Observation}
\newtheorem{corollary}[theorem]{Corollary}
\theoremstyle{definition}

\theoremstyle{remark}

\usepackage[textsize=tiny]{todonotes}

\icmltitlerunning{LoRA Training in the NTK Regime has No Spurious Local Minima}

\begin{document}

\twocolumn[
\icmltitle{LoRA Training in the NTK Regime has No Spurious Local Minima}



\icmlsetsymbol{equal}{*}

\begin{icmlauthorlist}
\icmlauthor{Uijeong Jang}{yyy}
\icmlauthor{Jason D. Lee}{xxx}
\icmlauthor{Ernest K. Ryu}{zzz}
\end{icmlauthorlist}

\icmlaffiliation{yyy}{Department of Mathematical Sciences, Seoul National University}
\icmlaffiliation{xxx}{Department of Electrical and Computer Engineering, Princeton University}
\icmlaffiliation{zzz}{Department of Mathematics, University of California, Los Angeles}

\icmlcorrespondingauthor{Ernest Ryu}{eryu@math.ucla.edu}

\icmlkeywords{Machine Learning, ICML}

\vskip 0.3in
]



\printAffiliationsAndNotice{}  

\begin{abstract}
Low-rank adaptation (LoRA) has become the standard approach for parameter-efficient fine-tuning of large language models (LLM), but our theoretical understanding of LoRA has been limited. In this work, we theoretically analyze LoRA fine-tuning in the neural tangent kernel (NTK) regime with $N$ data points, showing: (i) full fine-tuning (without LoRA) admits a low-rank solution of rank $r\lesssim \sqrt{N}$; (ii) using LoRA with rank $r\gtrsim \sqrt{N}$ eliminates spurious local minima, allowing (stochastic) gradient descent to find the low-rank solutions; (iii) the low-rank solution found using LoRA generalizes well.
\end{abstract}

\section{Introduction}
\label{s:introduction}
The modern methodology of using large language models involves (at least) two phases: self-supervised pre-training on a large corpus followed by supervised fine-tuning to the downstream task. As large language models have grown in scale, pre-training has become out of reach for research groups without access to enormous computational resources. However, supervised fine-tuning remains feasible for such groups. One key strategy facilitating this efficient fine-tuning is Parameter-Efficient Fine-Tuning (PEFT), which freezes most of the pre-trained model's weights while selectively fine-tuning a smaller number of parameters within an adapter module.
Among various PEFT methodologies, low-rank adaptation (LoRA) \cite{hu2021lora} has emerged as the standard approach. Given a pre-trained matrix $W_0\in\mathbb{R}^{m\times n}$, LoRA trains a low-rank update such that the forward pass evaluates
\[
W_0x+\Delta W x = W_0x+BA x
\]
where $r\ll \min(m,n)$,  $A\in\mathbb{R}^{r\times n}$ is initialized to be a random Gaussian, and $B\in\mathbb{R}^{m\times r}$ is initialized to be zero.


However, despite the widespread adoption of LoRA, our theoretical understanding of its mechanisms remains limited. One notable prior work is \citep{zeng2023expressive}, which analyzes the expressive power of LoRA, showing that for any given function, there exist weight configurations for LoRA that approximate it. However, their work does not address whether LoRA can efficiently learn such configurations. Additionally, \citet{malladi2023kernel} experimentally demonstrated that under certain conditions, LoRA fine-tuning is nearly equivalent to a kernel regression, where the $A$ matrix provides random features and is essentially not trained. This regime neglects the possibility of the $A$ matrix learning new features and, consequently, leads to a LoRA rank requirement of $r\ge \Theta(1/\varepsilon^2)$, where $\varepsilon$ is an approximation tolerance, originating from the use of the Johnson--Lindenstrauss lemma \cite{johnson1984extension}. Crucially, LoRA's fundamental nature as a quadratic parameterization has not been considered in the prior analysis of trainability and generalizability.

\paragraph{Contribution.}
In this work, we theoretically analyze LoRA fine-tuning and present results on trainability and generalizability. We consider fine-tuning a deep (transformer) neural network with $K$-dimensional outputs using $N$ training (fine-tuning) data points. Assuming that training remains under the NTK regime, which we soon define and justify in Section~\ref{s:preliminaries}, we show the following. First, full fine-tuning (without LoRA) admits a rank-$r$ solution such that $\frac{r(r+1)}{2}\leq KN$. Second, using LoRA with rank $r$ such that $\frac{r(r+1)}{2}>KN$ eliminates spurious local minima, allowing (stochastic) gradient descent to find the low-rank solutions. Finally, the low-rank solution found using LoRA generalizes well.

\subsection{Prior works}
\paragraph{Theory of neural networks.}
The question of expressive power addresses whether certain neural networks of interest can approximate a given target function. Starting with the classical universal approximation theorems  \cite{cybenko1989approximation,hornik1990universal, barron1993universal}, much research has been conducted in this direction. \cite{delalleau2011shallow, bengio2011expressive,lu2017expressive, duan2023minimum}. These can be thought of as existence results.

The question of trainability addresses whether one can compute configurations of neural networks that approximate target functions. \citet{ghadimi2013stochastic, ge2015escaping,du2017gradient,jin2017escape} studied general convergence results of gradient descent and stochastic gradient descent. \citet{soltanolkotabi2018theoretical,du2018power, allen2019convergence, allen2019convergence2, du2019gradient, zou2018stochastic} studied the loss landscape of neural networks and showed that first-order methods converge to global minima under certain conditions. 

The question of generalization addresses whether neural networks trained on finite data can perform well on new unseen data. Classical learning theory \citep{koltchinskii2000rademacher,bartlett2002model,bousquet2002stability,hardt2016train,bartlett2017spectrally} uses concepts such as uniform stability or the Rademacher complexities to obtain generalization bounds. Generalization bounds in the context of modern deep learning often utilize different approaches \cite{wu2017towards,dinh2017sharp,zhang2021understanding}, we use the Rademacher complexity for obtaining our generalization results.

\paragraph{Neural tangent kernels.}
The theory of neural tangent kernel (NTK) concerns the training dynamics of certain infinitely wide neural networks. \citet{jacot2018neural} shows that the training of an infinitely wide neural network is equivalent to training a kernel machine. Various studies such as  \citep{arora2019exact,chen2020generalized} expand the NTK theory to more practical settings. Among these works, \citet{wei2022more} introduced the concept of empirical NTK (eNTK) and showed that kernel regression with pretrained initialization also performs well on real datasets, providing a background to utilize NTK theory in fine-tuning.

\paragraph{Theory of transformers and LLMs.}
As the transformer architecture \cite{vaswani2017attention} became the state-of-the-art architecture for natural language processing and other modalities, theoretical investigations of transformers have been pursued. Results include that transformers are universal approximators \cite{yun2019transformers}, that  
transformers can emulate a certain class of algorithmic instructions \cite{wei2022statistically,giannou2023looped}, and that weight matrices in transformers increase their rank during training \citep{boix2023transformers}.
Also, \citep{zhang2020adaptive,liu2020understanding} presents improved adaptive optimization methods for transformers.

\paragraph{PEFT methods and LoRA.}
Low-rank adaptation (LoRA) \cite{hu2021lora} has become the standard Parameter-Efficient Fine-Tuning (PEFT) method, and many variants of LoRA have been presented \cite{fu2023effectiveness,dettmers2023qlora,lialin2023relora}. 
LoRA has proven to be quite versatile and has been used for convolution layers \cite{yeh2023navigating} and for diffusion models \cite{ryu2023low,smith2023continual,choilora}.

Theoretically, \citet{aghajanyan2020intrinsic} found an intrinsic low-rank structure is critical for fine-tuning language models, although this finding concerns full fine-tuning, not the setting that uses LoRA. Recently, \citet{zeng2023expressive} analyzed the expressive power of LoRA. However, we still lack a sufficient theoretical understanding of why LoRA is effective in the sense of optimization and generalization.


\paragraph{Matrix factorization.}
In this work, we utilize techniques developed in prior work on matrix factorization problems. \citet{bach2008matrix, haeffele14lowrank} established the sufficiency of low-rank parameterizations in matrix factorization problems, and their techniques have also been used in matrix completion \cite{ge2016matrix}, matrix sensing \cite{jin2023understanding}, and semidefinite programming \cite{bhojanapalli2018smoothed}.

\subsection{Organization}
Section~\ref{s:preliminaries} introduces the problem setting and reviews relevant prior notions and results. Section~\ref{s:existence} proves the existence of low-rank solutions. Section~\ref{s:optimization} proves LoRA has no spurious local minima and, therefore, establishes that (stochastic) gradient descent can find the low-rank global minima. Section~\ref{s:generalization} shows that the low-rank solution generalizes well. Finally, Section~\ref{s:experiment} presents simple experiments fine-tuning pre-trained models for different modalities. The experimental results validate our theory and provide further experimental insights.

\section{Problem setting and preliminaries}
\label{s:preliminaries}
We primarily consider the setup of pre-trained large language models fine-tuned with LoRA. However, our theory does generally apply to other setups that utilize pre-training and LoRA fine-tuning, such as diffusion models.


\paragraph{Matrix notation.}
For matrices $A$ and $B$, let $\|A\|_{*}$ denote the nuclear norm, $\|A\|_F$ the Frobenius norm, and $\langle A , B \rangle = \mathbf{tr}(A^{\intercal}B)$ the matrix inner product. We let $\mathbb{S}^{n}$ and $\mathbb{S}_{+}^{n}$ for the set of $n\times n$ symmetric and positive semi-definite matrices, respectively. Let $\mathcal{R}(\cdot)$ and $\mathcal{N}(\cdot)$ respectively denote the range and the null-space of a linear operator.

\paragraph{Neural network.}
Let $f_\Theta\colon\mathcal{X}\rightarrow\mathbb{R}^{K}$ be a neural network (e.g.,  a transformer-based model) parametrized by $\Theta$, where $\mathcal{X}$ is the set of data (e.g., natural language text) and $\mathbb{R}^{K}$ is the output (e.g.,  pre-softmax logits of tokens). 
$K$ is the output dimension of $f_\Theta$, where $K=k$ for $k$-class classification, $K=1$ for binary classification, and $K$ is the dimension of the label $Y$ when using mean square error loss.
Assume the model has been pre-trained to $\Theta=\Theta_0$, i.e., the pre-trained model is $f_{\Theta_0}$.

Let $\mathbf{W}=(W^{(1)},\dots,W^{(T)})\subset \Theta$ be a subset of the weights (e.g., dense layers in QKV-attention) with size $W^{(i)}\in\mathbb{R}^{m_i\times n_i}$ for $i=1,\dots, T$ that we choose to fine-tune. Let $\mathbf{W}_0=(W^{(1)}_0,\dots,W^{(T)}_0)\subset \Theta_0$ be their corresponding pre-trained weights. With slight abuse of notation, write $f_\mathbf{W}$ to denote $f_\Theta$, where all parameters of $\Theta$ excluding $\mathbf{W}$ are fixed to their corresponding values in $\Theta_0$.

\paragraph{Fine-tuning loss.}
Assume we wish to fine-tune the pre-trained model with
\[
\{(X_i,Y_i)\}_{i=1}^N,
\]
where $N$ is the number of (fine-tuning) training data.
(In many NLP tasks, it is not uncommon to have $N< 100$.)
Denote $\boldsymbol{\delta}=(\delta^{(1)},\dots,\delta^{(T)})\subset \Theta$ to be the change of $\mathbf{W}$ after the fine-tuning, i.e., $f_{\mathbf{W}_0+\mathbf{\boldsymbol{\delta}}}$ is our fine-tuned model. We use the empirical risk
\[
\hat{\mathcal{L}}(\boldsymbol{\delta})=\frac{1}{N}\sum^N_{i=1}\ell(f_{\mathbf{W}_0+\boldsymbol{\delta}}(X_i),Y_i),
\]
with some loss function $\ell$. 
We assume $\ell(x,y)$ is convex, non-negative, and twice-differentiable with respect to $x$ for any $y$. (This assumption holds for the cross-entropy loss and the mean squared error loss.)
The empirical risk approximates the true risk
\[
\mathcal{L}(\boldsymbol{\delta})=
\mathop{\mathbb{E}}_{(X,Y)\sim \mathcal{P}}\big[\ell(f_{\mathbf{W}_0+\boldsymbol{\delta}}(X),Y)\big]
\]
with some data distribution $\mathcal{P}$.




\paragraph{NTK regime.}
Under the NTK regime (also referred to as the lazy-training regime), the change of the network can be approximated by its first-order Taylor expansion
\begin{equation}\label{eq:linearized}
f_{\mathbf{W_{0}}+\mathbf{\boldsymbol{\delta}}}(X) \approx f_{\mathbf{W_{0}}}(X)+ \langle \nabla f_{\mathbf{W_0}}(X) , \boldsymbol{\delta} \rangle  
\end{equation}
sufficiently well throughout (fine-tuning) training. To clarify, $f_{\mathbf{W_{0}}+\mathbf{\boldsymbol{\delta}}}(X)\in \mathbb{R}^K$, so the NTK regime requires the first-order Taylor expansion to be accurate for all coordinates:
\[
f^{(j)}_{\mathbf{W_{0}}+\mathbf{\boldsymbol{\delta}}}(X) \approx f^{(j)}_{\mathbf{W_{0}}}(X)+ \langle \nabla f^{(j)}_{\mathbf{W_0}}(X) , \boldsymbol{\delta} \rangle,
\]
where $f^{(j)}_{\mathbf{W}}$ is the $j$-th coordinate of $f_{\mathbf{W}}$ for $j=1,\dots,K$.

The NTK regime is a reasonable assumption in fine-tuning if $\boldsymbol{\delta}$ is small, and this assertion is supported by the empirical evidence of \citep{malladi2023kernel}. This prior work provides extensive experiments on various NLP tasks to validate that fine-tuning happens within the NTK regime for many, although not all, NLP tasks.
\begin{observation}[\citet{malladi2023kernel}]
When prompt-based fine-tuning \cite{schick2020exploiting,gao2020making} is used, fine-tuning a pre-trained language model stays within the NTK regime.
\end{observation}

Motivated by this empirical observation, we define linearized losses
\[
\!
\hat{L}(\boldsymbol{\delta})=\frac{1}{N}
\sum^N_{i=1}
\ell\left(
f^{}_{\mathbf{W}_0}(X_i)+\langle \nabla f_{\mathbf{W_0}}(X_i), \boldsymbol{\delta}\rangle 
, Y_i\right)\approx\hat{\mathcal{L}}(\boldsymbol{\delta})
\]
and
\[
\!
{L}(\boldsymbol{\delta})=
\!\!\!\!\!\!
\mathop{\mathbb{E}}_{(X,Y)\sim \mathcal{P}}\Big[
\ell\left(
f^{}_{\mathbf{W}_0}(X_i)+\langle \nabla f_{\mathbf{W_0}}(X_i), \boldsymbol{\delta}\rangle 
, Y_i\right)\Big]\approx\mathcal{L}(\boldsymbol{\delta}).
\]

\paragraph{LoRA.}
We use the low-rank parameterization 
\[
\delta^{(i)}=u^{(i)}(v^{(i)})^\intercal\in \mathbb{R}^{m_i\times n_i},
\]
where $u^{(i)}\in\mathbb{R}^{m_i\times r}, v^{(i)}\in\mathbb{R}^{n_i \times r}$, for $i\in\{1,\cdots,T\}$.
Under the NTK regime, the empirical risk can be approximated as
\[
\hat{L}(\mathbf{u}\mathbf{v}^{\intercal})= \frac{1}{N}
\sum^N_{i=1}
\ell\left(
f^{}_{\mathbf{W}_0}(X_i)+\langle \mathbf{G}^{}(X_i), \mathbf{u}\mathbf{v}^{\intercal}\rangle 
 , Y_i\right),
\]
where
\begin{gather*}
\mathbf{u}=
\begin{bmatrix}
u^{(1)}\\\vdots\\u^{(T)}
\end{bmatrix}\in \mathbb{R}^{m\times r}
,\qquad
\mathbf{v}=
\begin{bmatrix}
v^{(1)}\\ \vdots \\v^{(T)}
\end{bmatrix}\in \mathbb{R}^{n\times r}
\end{gather*}
with $m=\sum_{i=1}^{T}m_i$ and $n=\sum_{i=1}^{T}n_i$,
and
\[
\mathbf{G}^{}(X_i)=\mathrm{diag}\left(
\nabla_{W^{(1)}}f^{}_{\mathbf{W}_0}(X_i),\dots,
\nabla_{W^{(T)}}f^{}_{\mathbf{W}_0}(X_i)\right)
\]
is an collection of $K$ $m\times n$ block diagonal matrices. To clarify, $\mathbf{G}^{}(X_i)\in \mathbb{R}^{K\times m\times n}$, so $\langle \mathbf{G}^{}(X_i),\mathbf{u}\mathbf{v}^{\intercal}\rangle \in \mathbb{R}^K$ should be interpreted as $K$ inner products of $m\times n$ matrices where each matrices correspond to each coordinates of $f$. 
More specifically, $\mathbf{G}^{(j)}(X_i)\in\mathbb{R}^{m\times n}$
and
\[
\big(\langle \mathbf{G}^{}(X_i),\mathbf{u}\mathbf{v}^{\intercal}\rangle\big)_j=\langle \mathbf{G}^{(j)}(X_i),\mathbf{u}\mathbf{v}^{\intercal}\rangle
\]
for $j=1,\dots,K$. 
Note that $\hat{L}(\mathbf{u}\mathbf{v}^{\intercal})$ under the NTK regime is non-convex in $(\mathbf{u},\mathbf{v})$ so SGD-training does not converge to the global minimizer, in general.


\paragraph{Weight decay on LoRA is nuclear norm regularization.}
The LoRA training of optimizing $\hat{L}$ is often conducted with weight decay \cite{hu2021lora, dettmers2023qlora}, which can be interpreted as solving
\begin{equation*}
\begin{array}{ll}
\underset{\mathbf{u},\,\mathbf{v}}{\mbox{minimize}}
&
\hat{L}(\mathbf{u}\mathbf{v}^\intercal)+\frac{\lambda}{2}\|\mathbf{u}\|_F^2+\frac{\lambda}{2}\|\mathbf{v}\|_F^2,
\end{array}
\end{equation*}
with regularization parameter $\lambda\ge 0$. This problem is equivalent to the rank-constrained nuclear-norm regularized problem
\begin{equation*}
\begin{array}{ll}
\underset{\boldsymbol{\delta},\, \mathrm{rank}\boldsymbol{\delta}\leq r} {\mbox{minimize}}
&\hat{L}_{\lambda}(\boldsymbol{\delta})\triangleq\hat{L}(\boldsymbol{\delta})+\lambda\|\boldsymbol{\delta}\|_{*}.
\end{array}
\end{equation*}
This is due to the following lemma.
\begin{lemma}[Lemma 5.1 of \citep{recht2010guaranteed}]\label{lem:regular}
Let $r>0$. For $\boldsymbol{\delta}\in \mathbb{R}^{m\times n}$ such that $\mathrm{rank}(\boldsymbol{\delta})\leq r$,
\[
\!\!
\|\boldsymbol{\delta}\|_{*}=\frac{1}{2}\underset{\mathbf{u}\mathbf{v}^\intercal = \boldsymbol{\delta}}{\min}
\{\|\mathbf{u}\|_F^2+\|\mathbf{v}\|_F^2
\,|\,
\mathbf{u}\in \mathbb{R}^{m\times r},\,\mathbf{v}\in \mathbb{R}^{n\times r}
\}.\]
\end{lemma}
(The connection between weight decay on Burer--Monteiro style low-rank factorization and nuclear norm regularization has been previously in different contexts not directly related to LoRA \cite{cabral2013unifying,pilanci2020neural}.)

\paragraph{Second-order stationary points.}

Let $\hat{L}\colon\mathbb{R}^{m\times n}\rightarrow\mathbb{R}$ be twice-continuously differentiable. We say $U\in \mathbb{R}^{m\times n}$ is a (first-order) \emph{stationary} point if
\[
\nabla \hat{L}(U)=\mathbf{0}.
\]
We say $U\in \mathbb{R}^{m\times n}$ is a \emph{second-order stationary point} (SOSP) if
\[
\nabla \hat{L}(U)=\mathbf{0},\qquad
\nabla^2 \hat{L}(U)[V,V]\geq 0,
\]
for any direction $V\in \mathbb{R}^{m\times n}$. We say $U$ is \emph{strict saddle} if $U$ is a first- but not second-order stationary point. Lastly, we say $U\in \mathbb{R}^{m\times n}$ is a \emph{local minimum} if there exists an open ball $B$ that contains $U$ and 
\[
\hat{L}(U)\leq \hat{L}(U')
\]
for any $U'\in B$. It follows that a local minimum is an SOSP.

The following results, roughly speaking, establish that (stochastic) gradient descent only converges to SOSPs when a loss function is twice-continuously differentiable. 
\begin{theorem}[Theorem 4.1 of \citep{lee2016gradient}]
\label{thm:gdconverges}
Gradient descent on twice-differentiable function with random initialization, almost surely, does not converge to strict saddle points.
I.e., if gradient descent converges, it converges to an SOSP, almost surely.
\end{theorem}
\begin{theorem}[Informal, Theorem 1 of \citep{ge2015escaping}]
\label{thm:sgdconverges}
Stochastic gradient descent with noise on twice-differentiable strict saddle function (i.e., every stationary point is either a local minimum or a strict saddle) does not converge to strict saddle points with high probability.
I.e., if stochastic gradient descent with noise converges, it converges to an SOSP with high probability.
\end{theorem}
Therefore, if we can show that all SOSPs are global minima in our setup of interest, then (stochastic) gradient descent will only converge to global minima.

\begin{figure*}[!ht]
\vskip 0.2in
\begin{center}
 \centering
        \begin{subfigure}[b]{0.33\textwidth}
            \centering
            \includegraphics[width=\textwidth]{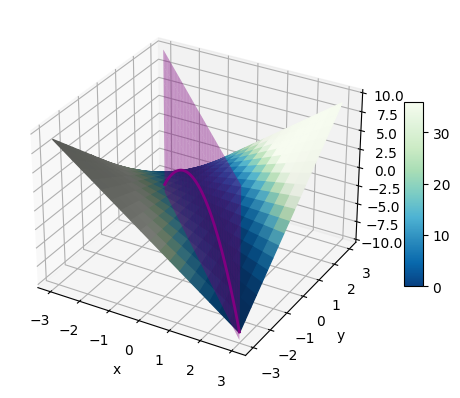}
        \end{subfigure}
        \begin{subfigure}[b]{0.33\textwidth}  
            \centering 
            \includegraphics[width=\textwidth]{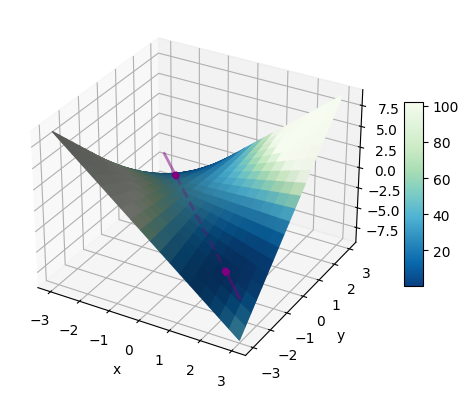}
        \end{subfigure}
                \begin{subfigure}[b]{0.33\textwidth} 
            \centering 
            \includegraphics[width=\textwidth]{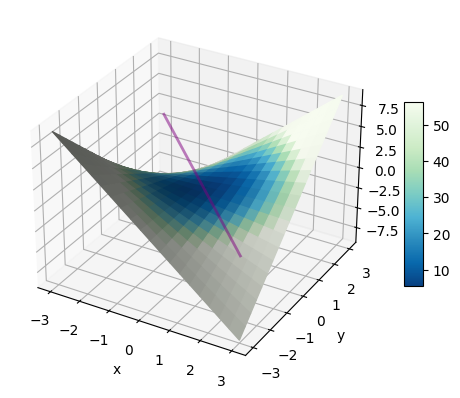}
        \end{subfigure}   
\caption{Geometric intuition of Theorem~\ref{thm:existence}. The three dimensional space describes the space of 2 by 2 matrices $\begin{bmatrix}
    1 & x \\ y & z
\end{bmatrix}$. The surface $z=xy$ represents the rank 1 matrices. The blue region on the surface correspond to the region of smaller objective values, and the set of global minima are depicted with purple. \textbf{(Left)} Plot of \eqref{eq:ex1} with $N=1$. The set of global minima is a plane, and the intersection with the surface  $z=xy$ (curve) is the set of rank-$1$ global minima. \textbf{(Middle)} Plot of \eqref{eq:ex2} with $N=2$. the set of global minima is a line, and the intersection with the surface (two dots) is the set of rank 1 global minima. \textbf{(Right)} Plot of \eqref{eq:ex3} with $N=3$. The set of global minima is a line, and there is no intersection with the surface, i.e., there is no global minimum of rank-$1$ but admits a rank-$2$ global minima.}
\label{fig:fig1}
\end{center}
\vskip -0.2in
\end{figure*}

\section{Low-rank solution exists}
\label{s:existence}
In this section, we show that full fine-tuning in the NTK regime admits a low-rank solution of rank $r\lesssim \sqrt{N}$.
The existence of a low-rank solution provides theoretical legitimacy to using the low-rank parameterization of LoRA, which, of course, can only find low-rank solutions.

\begin{theorem}\label{thm:existence}
Let $\lambda\geq 0$. Assume $\hat{L}_{\lambda}(\boldsymbol{\delta})$ has a global minimizer (not necessarily unique). Then there is a rank-$r$ solution such that $\frac{r(r+1)}{2}\leq KN$.
\end{theorem}

The assumption that  $\hat{L}_{\lambda}(\boldsymbol{\delta})$ has a global minimum is very mild; it is automatically satisfied if $\lambda>0$.
When $\lambda=0$, the assumption holds if $\ell$ is the mean squared error loss. 

The inspiration for Theorem~\ref{thm:existence} comes from the classical results of \citep{barvinok1995problems,pataki1998,pataki2000geometry} that establish that semi-definite programs (which have symmetric positive semi-definite matrices as optimization variables) admit low-rank solutions. We clarify that Theorem~\ref{thm:existence} does not require $\boldsymbol{\delta}$ to be symmetric nor any notion of ``semi-definiteness'' ($\boldsymbol{\delta}$ is not even square).

\begin{proof}[Proof sketch of Theorem~\ref{thm:existence}]
We quickly outline the key ideas of the proof while deferring the details to Appendix~\ref{a:existence}. 

We can show that finding
$\boldsymbol{\delta}^\star_\lambda\in\argmin_{\boldsymbol{\delta}}\hat{L}_{\lambda}({\boldsymbol{\delta}})$ with $\mathrm{rank}(\boldsymbol{\delta}^\star_\lambda)=r$ is equivalent to finding a rank-$r$ global minimum of  $F\colon\mathbb{S}_{+}^{(m+n)} \rightarrow \mathbb{R}$ where
\[
F(Z) = \hat{L}(\bar{Z})+\frac{\lambda}{2}\mathbf{tr}({Z})
\]
and $\bar{Z}=Z[1:m, m+1:m+n]\in\mathbb{R}^{m\times n}$. I.e., $\bar{Z}$ is a off-diagonal submatrix of $Z$ such that
\begin{equation}\label{eq:submatrix}
Z=\begin{bmatrix}
    * & \bar{Z} \\ \bar{Z}^\intercal & * 
\end{bmatrix}.
\end{equation} Now suppose $Z^\star\in\mathbb{S}_{+}^{(m+n)}$ is a global minimizer of $F$.
Define $\mathcal{S}(Z^\star)\triangleq\{Z\in\mathbb{S}_{}^{(m+n)}\colon \mathcal{R}(Z)\subseteq\mathcal{R}(Z^\star)\}$ and a linear operator $\mathcal{A}\colon \mathbb{S}^{(m+n)}\rightarrow\mathbb{R}^{KN}$ as 
\[
\mathcal{A}(Z)_{ij}= \langle \mathbf{G}^{(j)}(X_i), \bar{Z} \rangle, \qquad 1\leq i \leq N,\quad  1\leq j \leq K. 
\]
Now let $\mathrm{rank}(Z^\star)=r$ and assume 
\[
\{\mathbf{0}\} = \mathcal{S}(Z^\star) \cap \mathcal{N}(\mathcal{A}).
\]
Then by dimension counting, we have the following inequality.
\begin{align*}
    \!\!\! 0&=\mathrm{dim}\mathcal{S}(Z^\star)+\mathrm{dim}\mathcal{N}(\mathcal{A})-\mathrm{dim}(\mathcal{S}(Z^\star)+\mathcal{N}(\mathcal{A}))\\
    &=\mathrm{dim}\mathcal{S}(Z^\star)+\mathrm{dim}(\mathbb{S}^{(m+n)})-\mathrm{dim}\mathcal{R}
(\mathcal{A})\\
&-\mathrm{dim}(\mathcal{S}(Z^\star)+\mathcal{N}(\mathcal{A}))\\
 &= \mathrm{dim}\mathcal{S}(Z^\star)-KN+\mathrm{dim}(\mathbb{S}^{(m+n)})\\
 &\qquad -\mathrm{dim}(\mathcal{S}(Z^\star)+\mathcal{N}(\mathcal{A}))\\
&= \mathrm{dim}\mathcal{S}(Z^\star)-KN+\mathrm{dim}(\mathcal{S}(Z^\star)^{\perp}\cap\mathcal{R}(\mathcal{A}))\\
&\geq\mathrm{dim}\mathcal{S}(Z^\star)-KN
\end{align*}
If there exists nonzero $Z\in\mathbb{S}^{(m+n)}$
such that $Z\in \mathcal{S}(Z^\star) \cap \mathcal{N}(\mathcal{A})$, then we can show that there exists nonzero $t\in\mathbb{R}$ such that $Z^\star+tZ$ is also a global minimizer of $F$ with strictly lower rank. Replace $Z^\star$ with $Z^\star+tZ$ and repeat this process until we find a solution $Z^\star$ with
\[
\{\mathbf{0}\} = \mathcal{S}(Z^\star) \cap \mathcal{N}(\mathcal{A}).
\]
Together with the fact that $\mathrm{dim}\mathcal{S}(Z^\star)=\frac{r(r+1)}{2}$, we have the desired result.
\end{proof}

\paragraph{Illustration of Theorem~\ref{thm:existence}.}
The following toy example illustrates the geometric intuition of Theorem~\ref{thm:existence}. Let $\ell$ be the mean square error loss, $K=1$, $\boldsymbol{\delta}=\begin{bmatrix}
    w & x \\ y & z 
\end{bmatrix}$, and $\lambda=0$ (no regularization). Then consider the following objective functions each for $N=1$, $2$, and $3$:
\begin{gather}\label{eq:ex1}\tag{a}
        \hat{L}_{0}(\boldsymbol{\delta})=(x+y)^2
    \\
    \label{eq:ex2}\tag{b}
        \hat{L}_{0}(\boldsymbol{\delta})=\frac{1}{2}(z+4)^2+\frac{1}{2}(x+y)^2
    \\\label{eq:ex3}\tag{c}
        \!\!\!\!\hat{L}_{0}(\boldsymbol{\delta})=\frac{1}{3}(w-1)^2+\frac{1}{3}(z-4)^2+\frac{1}{3}(\sqrt{3}x+\sqrt{3}y)^2
    \end{gather}
The set of low-rank (rank-$1$) solutions for the three objectives are depicted in Figure~\ref{fig:fig1}. 

\section{GD and LoRA finds low-rank solution}\label{s:optimization}
In this section, we show that the optimization landscape with LoRA in the NTK regime has no spurious local minima if the LoRA parameterization uses rank $r\gtrsim \sqrt{N}$ and if we consider an $\varepsilon$-perturbed loss. This implies that optimizers such as stochastic gradient descent only converge to the low-rank global minimizers.
 
\begin{theorem}\label{thm:second}
    Let $\lambda\geq 0$. Assume $\hat{L}_{\lambda}(\boldsymbol{\delta})$ has a global minimizer (not necessarily unique) and $\frac{r(r+1)}{2}>KN$. Consider the perturbed loss function $\hat{L}_{\lambda,P}$ defined as 
    \[
\hat{L}_{\lambda,P}(\mathbf{u},\mathbf{v})\triangleq \hat{L}(\mathbf{u}\mathbf{v}^{\intercal})+\frac{\lambda}{2}\|\mathbf{u}\|_F^2+\frac{\lambda}{2}\|\mathbf{v}\|_F^2+\langle P, QQ^\intercal \rangle,
    \]
    where $Q=\begin{bmatrix}
        {\mathbf{u}} \\ \mathbf{v}
    \end{bmatrix}\in\mathbb{R}^{(m+n)\times r}$ and $P\in\mathbb{S}_{+}^{(m+n)}$ is positive semi-definite.
Then, for almost all nonzero $P$ (with respect to the Lebesgue measure on $\mathbb{S}_{+}^{(m+n)}\subset\mathbb{S}^{(m+n)}\cong\mathbb{R}^{\frac{(m+n)(m+n+1)}{2}}$), all SOSPs of $\hat{L}_{\lambda,P}$ are global minimizers of $\hat{L}_{\lambda,P}$.
\end{theorem}
To clarify, the conclusion that `all SOSPs are global minimizers' holds with probability $1$ even if the distribution of $P$ is supported on $\{P\in\mathbb{S}_{+}^{(m+n)} : \|P\|\le \varepsilon\}$ for arbitrarily small $\varepsilon>0$. In the practical LoRA fine-tuning setup where no perturbation is used and $P=0$ is set deterministically, Theorem~\ref{thm:second} does not apply. However, we can nevertheless interpret the result of Theorem~\ref{thm:second} to show that LoRA fine-tuning \emph{generically} has no spurious local minima.

If we do use a randomly generated small perturbation $P$ so that Theorem~\ref{thm:second} applies, the solution to the perturbed problem with small $P$ does not differ much from that of the unperturbed problem with $P=0$ in the following sense.
\begin{corollary}
\label{cor:opt}
Consider the setup of Theorem~\ref{thm:second} and let $\varepsilon>0$. Assume $\boldsymbol{\delta}^\star_\lambda \in \argmin_{\boldsymbol{\delta}}\hat{L}_{\lambda}(\boldsymbol{\delta})$. Assume $P$ is randomly sampled with a probability distribution supported in
\[
\{P\in\mathbb{S}_{+}^{(m+n)}: \|P\|_{F}< \varepsilon\}
\]
and is absolutely continuous with respect to the Lebesgue measure on $\mathbb{S}^{(m+n)}\cong\mathbb{R}^{\frac{(m+n)(m+n+1)}{2}}$.
Then for any SOSP $ (\hat{\mathbf{u}},\hat{\mathbf{v}})$ of $\hat{L}_{\lambda,P}$
        \begin{align*}            
        \hat{L}_{\lambda}(\hat{\mathbf{u}}\hat{\mathbf{v}}^\intercal)&\leq \hat{L}(\boldsymbol{\delta}^\star_\lambda)+\lambda\|\boldsymbol{\delta}^\star_\lambda\|_{*}  + 2\varepsilon\|\boldsymbol{\delta}^\star_\lambda\|_{*}\\
        &=\min_{\boldsymbol{\delta}} \hat{L}_{\lambda}(\boldsymbol{\delta})+ 2\varepsilon\|\boldsymbol{\delta}^\star_\lambda\|_{*}.
        \end{align*}
\end{corollary}
I.e., if $(\hat{\mathbf{u}},\hat{\mathbf{v}})$ is an SOSP (and thus a global minimizer by Theorem~\ref{thm:second}) of the perturbed loss $\hat{L}_{\lambda,P}$, then it is an $\varepsilon$-approximate minimizer of the unperturbed loss $\hat{L}_{\lambda}$.

So if $\frac{r(r+1)}{2}>KN$, then Theorem~\ref{thm:gdconverges}, Theorem~\ref{thm:sgdconverges}, and Corollary~\ref{cor:opt} together establish that (stochastic) gradient descent finds a $\hat{\mathbf{u}}\hat{\mathbf{v}}^\intercal$ such that its unperturbed empirical risk is $\varepsilon$-close to the the minimum unperturbed empirical risk.

\subsection{Proof outlines}

The proof is done by continuing our analysis of global minimum of $\hat{L}_{\lambda}(\boldsymbol{\delta})$. Given that low-rank solution exists, which we proved in the previous section, recall that LoRA training with weight decay is equivalent to solving 
\begin{equation*}
    \argmin_{\mathbf{u},\mathbf{v}} \hat{L}(\mathbf{u}\mathbf{v}^\intercal)+\frac{\lambda}{2}\|\mathbf{u}\|_F^2+\frac{\lambda}{2}\|\mathbf{v}\|_F^2.
\end{equation*}
In this section, we relate SOSPs with global minimum, which opens the chance to find a global minimum by using gradient-based optimization methods. 
We start the analysis from the following lemma, which is a prior characterization of SOSPs in the matrix factorization. 
\begin{lemma}\label{lem:haeffele14}(Theorem 2 of \citep{haeffele14lowrank})
Let $G\colon \mathbb{S}_{+}^{(m+n)}\rightarrow\mathbb{R}$ be a twice differentiable convex function with compact level sets, $H\colon \mathbb{S}_{+}^{(m+n)}\rightarrow\mathbb{R}$ be a proper convex lower semi-continuous function, and $r>0$. If the function $F\colon U \mapsto G(UU^\intercal)+H(UU^\intercal)$ defined
over matrices $U\in\mathbb{R}^{(m+n)\times r}$ has a second order staionary point at a rank-deficient matrix $U$, then $UU^\intercal$ is a global minimum of $G+H$.
\end{lemma}

We build our analysis upon Lemma~\ref{lem:haeffele14}. However, Lemma~\ref{lem:haeffele14} is not directly applicable to our setting since 
it requires that the SOSP must be rank-deficient. However, this can be effectively circumvented by employing a perturbed empirical risk:
\begin{equation*}
    \underset{\mathbf{u},\,\mathbf{v}}{\mbox{minimize}} \ \hat{L}(\mathbf{u}\mathbf{v}^\intercal)+\frac{\lambda}{2}\|\mathbf{u}\|_F^2+\frac{\lambda}{2}\|\mathbf{v}\|_F^2 + \langle P, QQ^\intercal \rangle ,
\end{equation*}
where $Q=\begin{bmatrix}
        {\mathbf{u}} \\ \mathbf{v}
    \end{bmatrix}$, and $P$ is a positive semi-definite matrix. Now we get the following lemma by applying Lemma~\ref{lem:haeffele14} to the perturbed empricial risk.
\begin{lemma}\label{lem:firstperturbed}
 Fix $\lambda\geq 0$. Assume $\hat{L}_{\lambda}(\boldsymbol{\delta})$ has a global minimum (not necessarily unique), $P\in\mathbb{S}_{+}^{(m+n)}$ is nonzero positive semi-definite, and $r>0$. If $\hat{Q}=\begin{bmatrix}
\hat{\mathbf{u}} \\ \hat{\mathbf{v}}
\end{bmatrix}\in\mathbb{R}^{(m+n)\times r}$ is a rank deficient SOSP of    
\[
\hat{L}_{\lambda,P}(\mathbf{u},\mathbf{v})= \hat{L}(\mathbf{u}\mathbf{v}^{\intercal})+\frac{\lambda}{2}\|\mathbf{u}\|_F^2+\frac{\lambda}{2}\|\mathbf{v}\|_F^2+\langle P, QQ^\intercal \rangle,
    \] then $\hat{Q}$ is a global minimum of $\hat{L}_{\lambda,P}(\mathbf{u},\mathbf{v})$. 
\end{lemma}
\begin{proof}
Define $G, H:\mathbb{S}_{+}^{(m+n)} \rightarrow \mathbb{R}$  to be
\[
G(X) = \frac{\lambda}{2}\mathbf{tr}({X})+\langle P, X \rangle,\quad
H(X)=\hat{L}(\bar{X})  
\]
where $\bar{X}$ is the off-diagonal submatrix of $X$ defined in \eqref{eq:submatrix}.
Note that $G$ has compact level set for every $\lambda \geq 0$ since $\mathbf{tr}(X)\geq0$ and $P,X$ are positive semi-definite, concluding that $\hat{Q}_{\lambda,P}$ is a global minimum of $F(Q)\triangleq G(QQ^\intercal)+H(QQ^\intercal)=\hat{L}_{\lambda,P}(\mathbf{u}, \mathbf{v})$.
\end{proof}
    We now give a detailed analysis of the proof of Theorem~\ref{thm:second}. The structure of the proof is inspired by the original work of \citet{pataki1998} and followed by \citet{burer2003nonlinear, boumal2016non, du2018power}.
    The proof uses an application of Sard's theorem of differential geometry. The argument is captured in Lemma~\ref{lem:sardthm}, and its proof is deferred to Appendix~\ref{a:sardthm}.

\begin{lemma}\label{lem:sardthm}
Let $\mathcal{M}$ be $m$-dimensional smooth manifold embedded in $\mathbb{R}^{d}$ and $V$ be a linear subspace of $\mathbb{R}^{d}$ with dimension $n$. If $m+n<d$, then the set
\[
\mathcal{M}+V=\{ p+v : p\in\mathcal{M}, v\in V\}
\]
has Lebesgue measure zero in $\mathbb{R}^{d}$.
\end{lemma}

\begin{proof}[Proof of Theorem~\ref{thm:second}]
We show that second-order stationary point $\hat{Q}_{\lambda,P}=\begin{bmatrix}
        \hat{\mathbf{u}} \\ \hat{\mathbf{v}}
    \end{bmatrix}$ is rank-deficient for almost all positive semi-definite $P$, then use  Lemma~\ref{lem:firstperturbed} to complete the proof. 
Denote $f^{(j)}$ for the $j$-th coordinate of $f$.
    For simplicity of notations, define 
    \[
    \hat{Y}_{i}^{(j)}\triangleq f^{(j)}_{\mathbf{W}_0}(X_i)+\langle \mathbf{G}^{(j)}(X_i), \mathbf{u}\mathbf{v}^{\intercal}\rangle,
    \] 
    and
    \[ 
    v_i^{(j)}\triangleq\frac{1}{N}\frac{\partial}{\partial \hat{Y}_{i}^{(j)} } \ell(\hat{Y}_{i}^{} , Y_i^{})
    \]
    for $1 \leq i \leq N$ and $1\leq j \leq K$, which depends on $\mathbf{u}$ and $\mathbf{v}$. 
    Then for $v=\{v_i^{(j)}\}\in\mathbb{R}^{KN}$ define 
    \[{S}({v})\triangleq \sum_{i=1}^{N}\sum_{j=1}^{K}v_i^{(j)}\mathbf{G}^{(j)}(X_i)\in\mathbb{R}^{m\times n}.
    \]
  Then by first-order gradient condition, we have
    \[
    \Bigg( 
    \underbrace{\begin{bmatrix} \mathbf{0} &{S}({v}) \\{S}({v})^\intercal & \mathbf{0} \end{bmatrix} + \lambda I + P }_{\triangleq M}\Bigg)\hat{Q}_{\lambda,P}=\mathbf{0}
    \]
 We observe that the range of $\hat{Q}_{\lambda,P}\in \mathbb{R}^{(m+n)\times r}$ is in the nullspace of $M\in\mathbb{S}^{(m+n)}$. We now suppose $\hat{Q}_{\lambda,P}$ has full rank, i.e., $\mathrm{rank}(\hat{Q}_{\lambda,P})=r$. 
 Hence, we have the following inequality:
    \[
    r=\mathrm{rank}(\hat{Q}_{\lambda,P})\leq \mathrm{dim} \ \mathcal{N}(M) \leq m+n
    \]
    Now for $r \leq s \leq m+n$ and $s\in\mathbb{Z}$, define 
    \begin{align*}
    \mathcal{A}_{s}=\Big\{ P : P= M-\lambda I,
    M \in\mathbb{S}^{(m+n)}, \mathrm{dim}\mathcal{N}({M})= s \Big\}.
    \end{align*}
    Then from  Proposition 2.1 of \citep{helmke1995critical}, $\mathcal{A}_s$ is a smooth manifold embedded in $\mathbb{R}^{\frac{(m+n)(m+n+1)}{2}}\cong\mathbb{S}^{(m+n)}$ with dimension
    \[
    \mathrm{dim}\mathcal{A}_{s} =\frac{(m+n+1)(m+n)}{2} - \frac{s(s+1)}{2}.
    \]
    Now by definition of $P$, we know that 
    \[
    P\in\bigcup_{s=r}^{m+n}\left(\mathcal{A}_{s}+\mathcal{R}(S)\right)
    \]
    where $``+"$ is the set-sum (Minkowski sum) and $\mathcal{R}(S)$ is the range of $S(v)$ in $\mathbb{R}^{\frac{(m+n)(m+n+1)}{2}}$ for any $v\in\mathbb{R}^{KN}$. The dimensions can be bounded by
   \begin{align*}
\mathrm{dim}\mathcal{A}_s&\leq \frac{(m+n)(m+n+1)}{2}-\frac{r(r+1)}{2}
\end{align*}
for $r\leq s\leq m+n$ and 
\begin{align*}
    \mathrm{dim}\mathcal{R}(S)\leq KN.
\end{align*}
Therefore given that $\frac{r(r+1)}{2}>KN$, we have 
\[
\mathrm{dim}\mathcal{A}_s+\mathrm{dim}\mathcal{R}(S) < \frac{(m+n)(m+n+1)}{2}.
\]
Then, by Lemma~\ref{lem:sardthm}, which is effectively an application of Sard's theorem, we can conclude $\mathcal{A}_{s}+\mathcal{R}(S)$ is a measure-zero set, and the finite union of such measure-zero sets is measure-zero.
This implies that every $P$ that makes $\hat{Q}_{\lambda,P}$ to be of full rank must be chosen from measure-zero subset of $\mathbb{S}_{+}^{(m+m)}\subset \mathbb{S}^{(m+n)}$. Therefore we may conclude that $\mathrm{rank}(\hat{Q}_{\lambda,P})<r$ for almost every nonzero positive semi-definite $P$.
\end{proof}
\begin{proof}[Proof of Corollary~\ref{cor:opt}]
    Assume $\boldsymbol{\delta}^\star_\lambda \in \argmin_{\boldsymbol{\delta}}\hat{L}_{\lambda}(\boldsymbol{\delta})$. We observe the following chain of inequalities.
    \begin{align*}
    \hat{L}(\hat{\boldsymbol{\delta}})+\lambda \| \hat{\boldsymbol{\delta}}\|_{*}&\leq
        \hat{L}(\hat{\mathbf{u}}\hat{\mathbf{v}}^\intercal)+\frac{\lambda}{2}\|\hat{\mathbf{u}}\|_F^2+\frac{\lambda}{2}\|\hat{\mathbf{v}}\|_F^2\\
        &\leq \hat{L}(\hat{\mathbf{u}}\hat{\mathbf{v}}^\intercal)+\frac{\lambda}{2}\|\hat{\mathbf{u}}\|_F^2+\frac{\lambda}{2}\|\hat{\mathbf{v}}\|_F^2 + \langle P , \hat{Q}\hat{Q}^\intercal \rangle \\
        &= \hat{L}_{\lambda,P}(\hat{\mathbf{u}},\hat{\mathbf{v}}),
        \end{align*}
where the first inequality of is from Lemma~\ref{lem:regular}, the second is from $P$ and $\hat{Q}\hat{Q}^\intercal$ being positive semi-definite. On the other hand, we can find $\mathbf{u}^\star$ and $\mathbf{v}^\star$ such that ${\boldsymbol{\delta}}^{\star}_{\lambda}=\mathbf{u}^{\star}\mathbf{v}^{\star\intercal}$ and $
\| {\boldsymbol{\delta}}^{\star}_{\lambda} \|_*=\frac{1}{2}
(\|\mathbf{u}^{\star}\|_F^2+\|\mathbf{v}^{\star}\|_F^2)
$ by using Lemma~\ref{lem:regular}. Now take $Q^\star =\begin{bmatrix}
    \mathbf{u}^\star \\ \mathbf{v}^\star
\end{bmatrix}$, then we get 
        \begin{align*}
         \hat{L}_{\lambda,P}({\mathbf{u}}^\star,{\mathbf{v}}^\star)&= \hat{L}(\boldsymbol{\delta}^{\star}_{\lambda})+\lambda\|\boldsymbol{\delta}^\star_{\lambda}\|_{*}+  \langle P ,{Q}^{\star}{Q}^{\star\intercal} \rangle \\
           &\leq\hat{L}(\boldsymbol{\delta}^{\star}_{\lambda})+\lambda\|\boldsymbol{\delta}^{\star}_{\lambda}\|_{*}+  \varepsilon \|Q^\star Q^{\star\intercal}\|_F  \\
                   &\leq \hat{L}(\boldsymbol{\delta}^{\star}_{\lambda})+\lambda\|\boldsymbol{\delta}^{\star}_{\lambda}\|_{*}+  \varepsilon \|Q^\star\|_F^2  \\
                    &= \hat{L}(\boldsymbol{\delta}^{\star}_{\lambda})+\lambda\|\boldsymbol{\delta}^{\star}_{\lambda}\|_{*}+  \varepsilon \|\mathbf{u}^\star\|_F^2+ \varepsilon \|\mathbf{v}^\star\|_F^2   \\
                     &= \hat{L}(\boldsymbol{\delta}^{\star}_{\lambda})+\lambda\|\boldsymbol{\delta}^{\star}_{\lambda}\|_{*}+  2\varepsilon \|\boldsymbol{\delta}^{\star}_{\lambda}\|_{*},
\end{align*}
where the first inequality is Cauchy--Schwartz inequality, and the second inequality is from sub-multiplicativity of $\|\cdot\|_F$. Moreover by Theorem~\ref{thm:second}, 
\[
\hat{L}_{\lambda,P}(\hat{{\mathbf{u}}}^\star,\hat{{\mathbf{v}}}^\star)\leq\hat{L}_{\lambda,P}({\mathbf{u}}^\star,{\mathbf{v}}^\star),
\]
and this happens for almost sure, since we sampled $P$ from a probability distribution which is absolutely continuous with respect to the Lebesgue measure on $\mathbb{R}^{\frac{(m+n)(m+n+1)}{2}}\cong\mathbb{S}^{(m+n)}$.
\end{proof}

 \begin{figure*}[!t]
        \centering
        \begin{subfigure}[b]{0.33\textwidth}
            \centering
            \includegraphics[width=\textwidth]{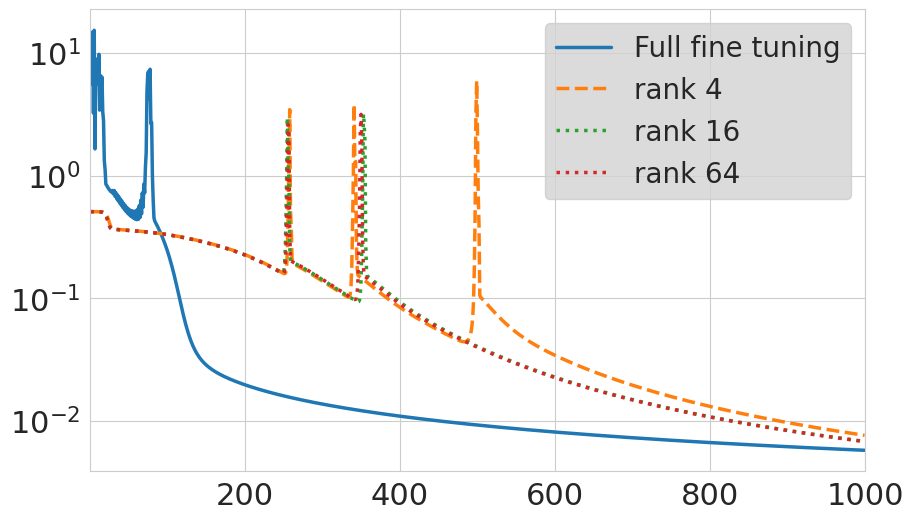}
            \caption{{SST-2}}    
            \label{fig:sst}
        \end{subfigure}
        \begin{subfigure}[b]{0.33\textwidth}  
            \centering 
            \includegraphics[width=\textwidth]{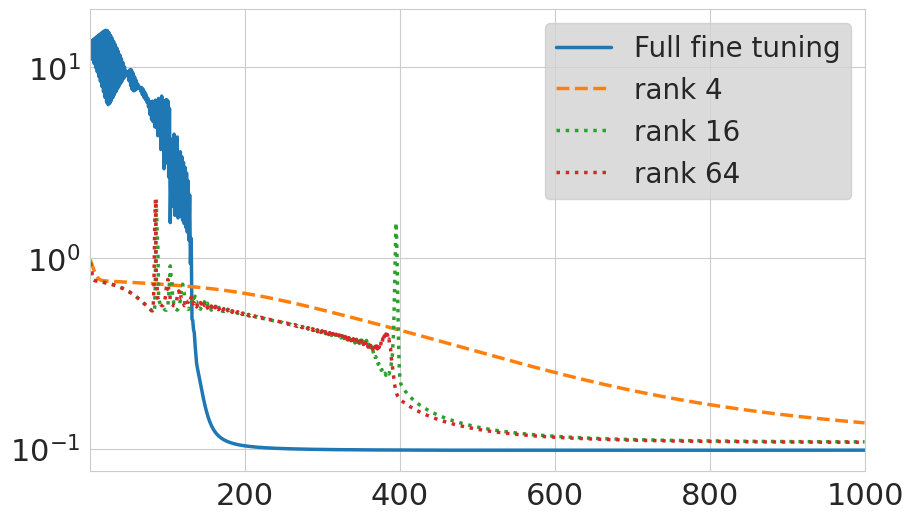}
            \caption
            {{ QNLI}}
            \label{fig:qnli}
        \end{subfigure}
                \begin{subfigure}[b]{0.33\textwidth} 
            \centering 
            \includegraphics[width=\textwidth]{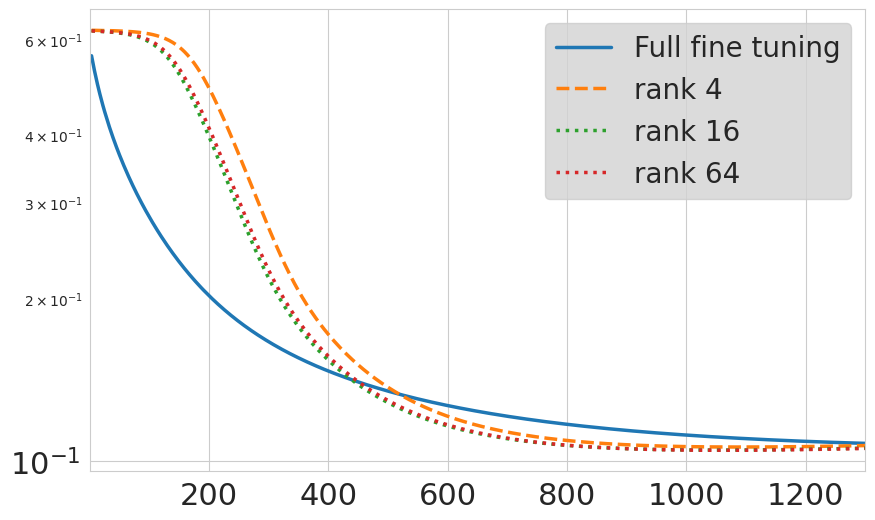}
            \caption
            {{MR}} 
            \label{fig:mr}
        \end{subfigure}
        \vskip\baselineskip
        \begin{subfigure}[b]{0.33\textwidth}   
            \centering 
            \includegraphics[width=\textwidth]{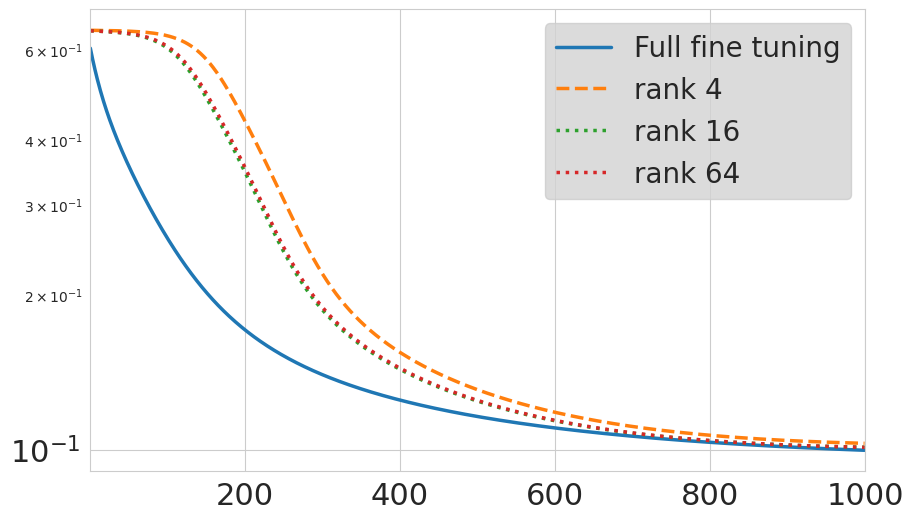}
            \caption
            {{CR }}  
            \label{fig:cr}
        \end{subfigure}
                \begin{subfigure}[b]{0.33\textwidth}  
            \centering 
            \includegraphics[width=\textwidth]{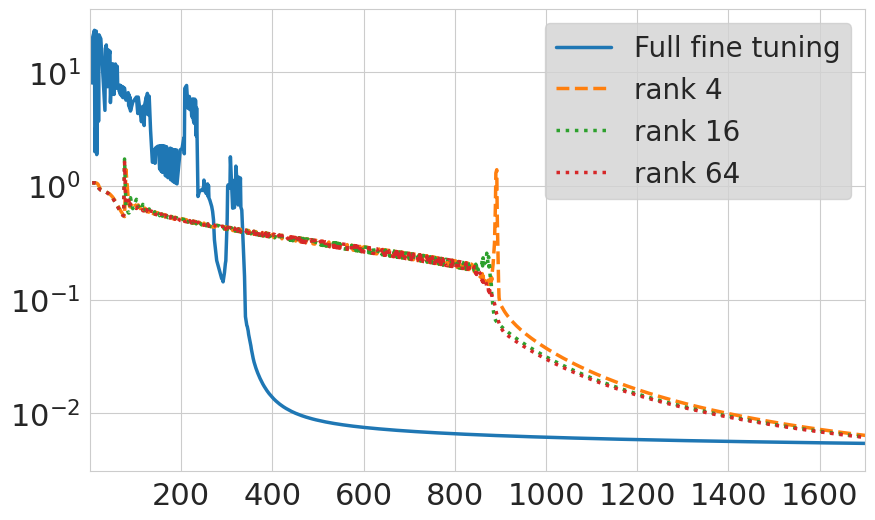}
            \caption
            {{QQP}} 
            \label{fig:qqp}
        \end{subfigure}
        \begin{subfigure}[b]{0.33\textwidth}   
            \centering 
            \includegraphics[width=\textwidth]{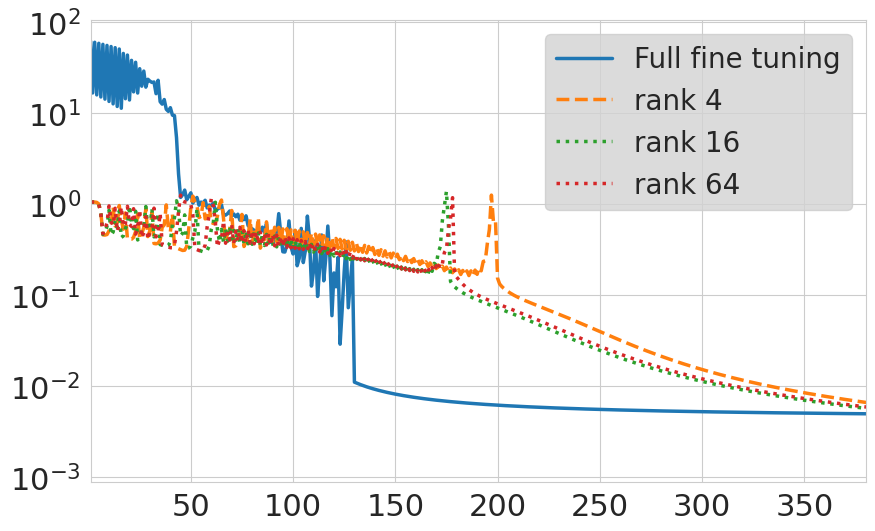}
            \caption
            {{Subj}}  
            \label{fig:subj}
        \end{subfigure}
        \caption{Training curves (training loss vs.\ epochs) on different NLP tasks.} 
        \label{fig:experiment}
    \end{figure*}

\section{Low-rank LoRA solution generalizes well}\label{s:generalization}
In this section, we establish a generalization guarantee for the low-rank solution obtained by minimizing the perturbed loss $\hat{L}_{\lambda,P}$ of Theorem~\ref{thm:second}.
For simplicity, we restrict the following main result to the cross-entropy loss. Generalization guarantees for general convex, non-negative, and twice continuously differentiable losses, are provided as Theorem~\ref{thm:gen} in Appendix~\ref{a:gen}.

\begin{theorem}\label{thm:gen-main}
Assume $\ell$ is cross-entropy loss. Assume the population risk $L$ has a minimizer (not necessarily unique) and denote it as  $\boldsymbol{\delta}^\star_\mathrm{true} \in \argmin_{\boldsymbol{\delta}}L(\boldsymbol{\delta})$.
Assume $\boldsymbol{\delta}^\star_{\mathrm{true}}\ne \mathbf{0}$.
For $1\leq j\leq K$, suppose $\|\mathbf{G}^{(j)}(X)\|_{F}\leq R$ almost surely with respect to the random data $X\sim \mathcal{P}$. 
Let $\varepsilon>0$, $\eta\in(0,1)$, and
\[
   \lambda=\frac{2(2+\varepsilon)\sqrt{K}R}{\sqrt{N}}\left(2+\sqrt{\log{\frac{1}{\eta}}}\right).
   \]
   Write ${\boldsymbol{\delta}}^\star_\lambda$ to denote a minimizer (not necessarily unique) of $\hat{L}_{\lambda}(\boldsymbol{\delta})$. Consider the setup of Corollary~\ref{cor:opt} with $P$ randomly sampled with a probability distribution supported in
\[
   \Big\{P\in\mathbb{S}_{+}^{(m+n)}: \|P\|_{F}<  \frac{\varepsilon{\lambda}\|\boldsymbol{\delta}^\star_{\mathrm{true}}\|_{*}}{2\|\boldsymbol{\delta}^\star_{\lambda}\|_{*}}\Big\}
\]
and is absolutely continuous with respect to the Lebesgue measure on $\mathbb{S}^{(m+n)}\cong\mathbb{R}^{\frac{(m+n)(m+n+1)}{2}}$.
   Let $(\hat{\mathbf{u}},\hat{\mathbf{v}})$ be an SOSP of $\hat{L}_{\lambda,P}$.  Then with probability greater than $1-\eta$, 
\[
\!\!\!
L(\hat{\mathbf{u}}\hat{\mathbf{v}}^\intercal) -L(\boldsymbol{\delta}^\star_{\mathrm{true}})<\|\boldsymbol{\delta}^\star_{\mathrm{true}}\|_{*}\frac{2(2+\varepsilon)^2\sqrt{K}R}{\sqrt{N}}\left(2+\sqrt{\log{\frac{1}{\eta}}}\right).
\]
\end{theorem}
In the context of fine-tuning, where the target task is closely related to the pre-training task, it is natural to assume that $\boldsymbol{\delta}^\star_{\mathrm{true}}$ in Theorem~\ref{thm:gen-main} is ``small". The proof, deferred to Appendix~\ref{a:gen}, utilizes standard arguments with Rademacher complexity.

\section{Experiments}\label{s:experiment}
In this section, we conduct simple experiments on fine-tuning linearized pre-trained models to validate our theory.\footnote{Code available at\\ \url{https://github.com/UijeongJang/LoRA-NTK}.} 


\paragraph{Experimental setup on NLP tasks.}
We use prompt-based fine-tuning \cite{schick2020exploiting,gao2020making} and consider the same architecture and dataset as in \citep{malladi2023kernel}, which empirically verifies that with prompt-based fine-tuning, the fine-tuning dynamics stay within the NTK regime. We present the results of six NLP tasks that were also considered in \citep{malladi2023kernel}: sentiment analysis (SST-2, MR, CR), natural language inference (QNLI), subjectivity (Subj), and paraphrase detection (QQP).
We optimize a linearized RoBERTa-base \cite{liu2019roberta} model with dataset of size 32 ($N=32$) with two labels ($K=2$) using cross entropy loss. With LoRA rank $r\geq 11$, our theory guarantees that no spurious local minima exist. For a baseline comparison, we also perform full fine-tuning (without LoRA) on the linearized model. The training curves are presented in Figure~\ref{fig:experiment}, and additional details are provided in Appendix~\ref{a:exp}. Results showing test accuracy are also presented in Appendix~\ref{a:exp}. 

\paragraph{Experimental setup on image and speech classification tasks.}
We use a pre-trained vision transformer \cite{dosovitskiy2020image} and fine-tune it on the bean disease dataset \cite{beansdata} to perform an image classification task with 3 labels. We use dataset of size 48 with three labels. Similar to our experiments on NLP tasks, we find that training curves converge to the same loss value, where the rates of convergence differ. 

\begin{figure*}[ht]
        \centering
        \begin{subfigure}[b]{0.33\textwidth}
            \centering
            \includegraphics[width=\textwidth]{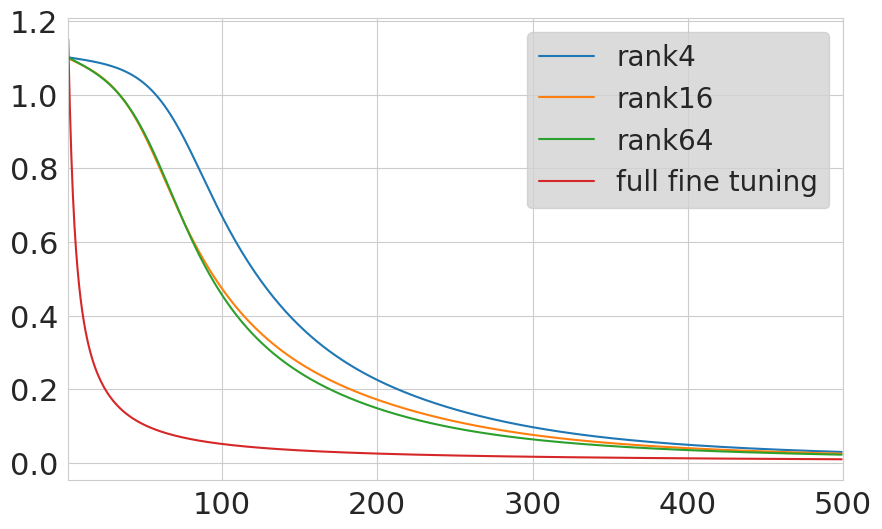}
            \caption{{Image classification}}    
            \label{fig:vision}
        \end{subfigure}
        \hspace{10mm}
        \begin{subfigure}[b]{0.33\textwidth}  
            \centering 
            \includegraphics[width=\textwidth]{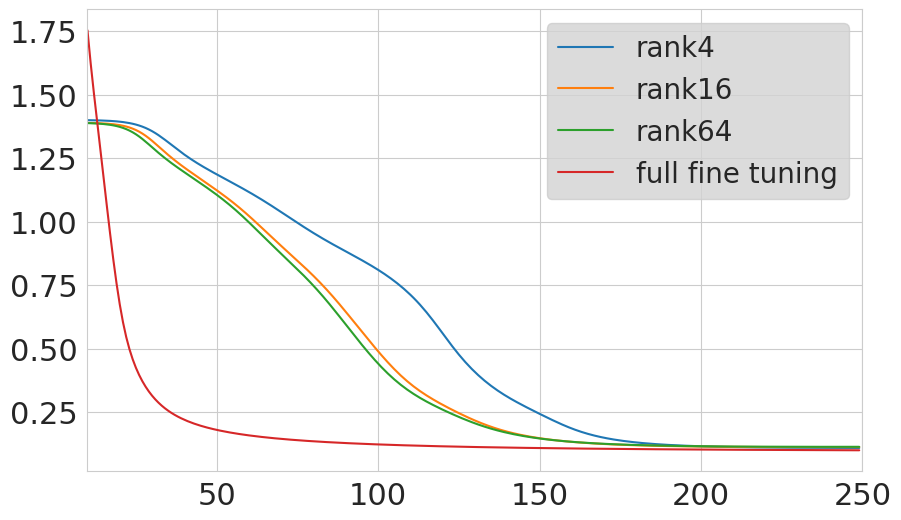}
            \caption
            {{Speech classification}}
            \label{fig:speech}
        \end{subfigure}
        \caption{Training curves (training loss vs.\ epochs) on image and speech classification tasks.} 
        \label{fig:experiment3}
    \end{figure*}
    
For speech classification, we use a pre-trained wav2vec2 \cite{baevski2020wav2vec} model and fine-tune it on a SUPERB dataset \cite{yang2021superb} to perform a speech classification task with 4 labels. We use a dataset of size 64 with four labels. We also find that the training curves converge to the same loss value. The details are the same as with the image classification task.

The training curves of both image and speech data are presented in Figure~\ref{fig:experiment3}, and additional details are provided in Appendix~\ref{a:exp}.

\paragraph{Empirical observation.}
The experiments validate our theory as the training curves converge to the same globally optimal loss value. However, we do observe that the  \emph{rates} of convergence differ. When the LoRA rank is higher or when full fine-tuning is performed and LoRA is not used, fine-tuning converges faster. Indeed, our theory ensures that spurious local minima do not exist, but it says nothing about how convex or favorable the landscape may or may not be. Our intuitive hypothesis is that using lower LoRA rank creates unfavorable regions of the loss landscape, such as plateaus or saddle points, and they slow down the gradient descent dynamics.

If this hypothesis is generally true, we face an interesting tradeoff: lower LoRA rank reduces memory cost and per-iteration computation cost but increases the number of iterations needed for convergence. Then, using a very low LoRA rank may be suboptimal not due to representation power, presence of spurious local minima, or poor generalization guarantees, but rather due to unfavorable flat training landscapes slowing down convergence. Exploring this phenomenon and designing remedies is an interesting direction for future work.

\section{Conclusion}
In this work, we present theoretical guarantees on the trainability and generalization capabilities of LoRA fine-tuning of pre-trained models.  Together with the work of \citet{zeng2023expressive}, our results represent a first step in theoretically analyzing the LoRA fine-tuning dynamics of pre-trained models by presenting guarantees (upper bounds). For future work, carrying out further refined analyses under more specific assumptions, relaxing the linearization/NTK regime assumption through a local analysis, better understanding the minimum rank requirement through lower bounds, and, motivated by the observation of Section~\ref{s:experiment}, analyzing the tradeoff between training rate and LoRA rank are exciting directions.

\section*{Acknowledgments}
UJ and EKR were supported by the Samsung Science and Technology Foundation (Project Number SSTF-BA2101-02) and the  National Research Foundation of Korea (NRF) Grant funded by the Korean Government (MSIP) [NRF-2022R1C1C1010010]. JDL acknowledges support of the NSF CCF 2002272, NSF IIS 2107304, and NSF CAREER Award 2144994.
We thank Jungsoo Kang for the discussion on the proof of Lemma \ref{lem:sardthm}. We also thank Jisun Park for providing valuable feedback.

\section*{Impact statement}
This paper presents work whose goal is to advance the field of Machine Learning. There are many potential societal consequences of our work, none which we feel must be specifically highlighted here.

\bibliography{icml_lorantk}
\bibliographystyle{icml2024}

\newpage
\appendix
\onecolumn
\section{Omitted proof of Theorem~\ref{thm:existence}}\label{a:existence}
Here, we explain the details in the proof of Theorem~\ref{thm:existence}. We first prove the equivalence of \begin{equation}\label{eq:p}\tag{P}
\underset{\boldsymbol{\ \delta}\in\mathbb{R}^{m\times n} }{\mathrm{minimize}}  \quad \hat{L}(\boldsymbol{\delta})+\lambda\|\boldsymbol{\delta}\|_{*}
\end{equation}
and
\begin{equation}\label{eq:q}\tag{Q}
 \underset{Z\in\mathbb{S}_{+}^{(m+n)}}{\mathrm{minimize}}  \quad   \hat{L}(\bar{Z})+\frac{\lambda}{2}\mathbf{tr}({Z})
\end{equation}
where $\bar{Z}=Z[1:m, m+1:m+n]\in\mathbb{R}^{m\times n}$. I.e., $\bar{Z}$ is a off-diagonal submatrix of $X$ such that
\begin{equation*}
Z=\begin{bmatrix}
    * & \bar{Z} \\ \bar{Z}^\intercal & * 
\end{bmatrix}.
\end{equation*} 
\begin{lemma}\label{lem:PQequiv}
The following two statements hold.
\begin{enumerate}
    \item Fix $\lambda\geq 0$ and suppose \eqref{eq:p} has a global minimizer (not necessarily unique). Let $\boldsymbol{\delta}^\star_{\lambda}\in\mathbb{R}^{m\times n}$ be a global minimizer of \eqref{eq:p}. Then there exists an $Z^\star_{\lambda}\in\mathbb{S}_{+}^{(m+n)}$ induced from
    $\boldsymbol{\delta}^\star_{\lambda}$ such that $Z^\star_{\lambda}$ is a global minimizer of \eqref{eq:q}, $\mathrm{rank}(Z^\star_{\lambda})=\mathrm{rank}(\boldsymbol{\delta}^\star_{\lambda})$, and has same objective value. 
        \item Fix $\lambda\geq 0$ and suppose \eqref{eq:q} has a global minimizer (not necessarily unique). Let $Z^\star_{\lambda}\in\mathbb{S}_{+}^{(m+n)}$ be a global minimum of \eqref{eq:q}. Then $\bar{Z^\star_{\lambda}}\in\mathbb{R}^{m\times n}$ is a global minimizer of \eqref{eq:p} such that $\mathrm{rank}(\bar{Z^\star_{\lambda}})=\min(m,n,\mathrm{rank}(Z^\star_{\lambda}))$ and has same objective value. 
\end{enumerate}
\begin{proof}
 We prove the two statements at once. Let $\boldsymbol{\delta}_{\lambda}^\star\in\mathbb{R}^{m\times n}$ be a global minimizer of \eqref{eq:p} and let $r=\mathrm{rank}(\boldsymbol{\delta}^\star_{\lambda})$. Then by Lemma~\ref{lem:regular}, there exists $\mathbf{u}\in\mathbb{R}^{m\times r}$ and $\mathbf{v}\in\mathbb{R}^{n\times r}$ such that $\|\boldsymbol{\delta}^\star_{\lambda}\|_{*}=\frac{1}{2}(\|\mathbf{u}\|_F^2+\|\mathbf{v}\|_F^2)$ and $\mathbf{u}\mathbf{v}^\intercal=\boldsymbol{\delta}^\star_{\lambda}$. Take 
    \[
    Z^\star_{\lambda}=\begin{bmatrix}
        \mathbf{{u}} \\ \mathbf{v}
    \end{bmatrix}\begin{bmatrix}
        \mathbf{u}^\intercal & \mathbf{v}^\intercal\end{bmatrix}=\begin{bmatrix}
    \mathbf{uu^\intercal} & \mathbf{uv^\intercal} \\ 
    \mathbf{vu^\intercal} & \mathbf{vv^\intercal} 
\end{bmatrix}\in\mathbb{S}_{+}^{(m+n)}.
    \]
    Then since
    \[
    \mathbf{tr}(Z^\star_{\lambda})=\|Z^\star_{\lambda}\|_{*}=\Big\|\begin{bmatrix}
        \mathbf{u} \\ \mathbf{v}
    \end{bmatrix}\Big\|_F^2 = \|\mathbf{u}\|_F^2+\|\mathbf{v}\|_F^2=2\|\boldsymbol{\delta}^\star_{\lambda}\|_{*},
    \]
$\eqref{eq:q}$ with $Z^\star_{\lambda}$ has the same objective value with $\eqref{eq:p}$ with $\boldsymbol{\delta}^\star_{\lambda}$ and $\mathrm{rank}(\boldsymbol{\delta}^\star_{\lambda})=\mathrm{rank}(Z^\star_{\lambda})=r$.   
Conversely, let $Z^\star_{\lambda}\in\mathbb{S}_{+}^{(m+n)}$ be a global minimizer of \eqref{eq:q} and let $\mathrm{rank}(Z^\star_{\lambda})=r$. Note that $r$ may be larger than $m$ or $n$. Then there exists $Q=\begin{bmatrix}
        \mathbf{{u}} \\ \mathbf{v}
    \end{bmatrix}\in\mathbb{R}^{(m+n)\times r}$ such that $QQ^\intercal=Z^\star_{\lambda}$. Then since
\[
\mathbf{tr}(Z^\star_{\lambda})=\|Z^\star_{\lambda}\|_{*}=\|Q\|_F^2=\|\mathbf{u}\|_F^2+\|\mathbf{v}\|_F^2\geq 2\|\mathbf{u}\mathbf{v}^\intercal\|_{*}=2\|\bar{Z^\star_{\lambda}}\|_{*},
\]
the objective value of \eqref{eq:p} with $\bar{Z^\star_{\lambda}}\in\mathbb{R}^{m\times n}$ has less than or equal to minimum objective value of \eqref{eq:q} and $\mathrm{rank}(\bar{Z^\star_{\lambda}})=\min(m,n,r)$.

If there exists $m\times n$ matrix whose objective value of \eqref{eq:p} is strictly less than the minimum objective value of \eqref{eq:q}, then we repeat the same step that was applied on $\boldsymbol{\delta}^\star_{\lambda}$ to induce a solution of \eqref{eq:q} with strictly less objective value, which is a contradiction. Conversely, if there exists positive semi-definite matrix of size $m+n$ whose objective value of \eqref{eq:q} is strictly less than the minimum objective value of \eqref{eq:p}, then we repeat the same step applied on ${Z^\star_{\lambda}}$ to induce a solution of \eqref{eq:p} with strictly less objective value, which is also a contradiction. Therefore if one of \eqref{eq:p} and \eqref{eq:q} has a global minimizer, the other must have a global minimizer with same objective value. 
\end{proof}
\end{lemma}
Next lemma states that if the rank of the global minimizer of $\eqref{eq:q}$ is sufficiently large, then we can find an another solution with strictly less rank. 
\begin{lemma}\label{lem:reduction}
    Suppose $X\in\mathbb{S}_{+}^{n}$ and let $Z\in\mathbb{S}^{n}$ be a nonzero symmetric matrix such that $\mathcal{R}(Z)\subseteq \mathcal{R}(X)$. Then there exists nonzero $t^*\in\mathbb{R}$ such that $X+t^* Z$ is positive semi-definite and $\mathrm{rank}(X+t^* Z)<\mathrm{rank}(X)$. 
\end{lemma}
\begin{proof}
      Let $r=\mathrm{rank}(X)$. Suppose $Q\in\mathbb{R}^{n\times r}$ is a matrix where its columns are basis to $\mathcal{R}(X)$. 
    Now suppose $\mu_{1}(Q^\intercal(X+t Z)Q)>0$ for all $t\in\mathbb{R}$ where $\mu_{1}(\cdot)$ denotes the smallest eigenvalue (note that $\mu_1(\cdot)$ is continuous). Then $Q^\intercal(X+tZ)Q\in\mathbb{S}^{r}$ should be positive definite for all $t$. For contradiction, take $v\in\mathcal{R}(Z)\subseteq\mathcal{R}(X)=\mathcal{R}(Q)$ to be an eigenvector of nonzero eigenvalue of $Z$. Since $v^\intercal X v>0$ and $v^\intercal {Z}v\ne 0$, there exists some $t^{}$ such that $v^\intercal (X+t^{} Z)v<0$. 
    Now take $w\in\mathbb{R}^{r}$ such that $Qw=v$. Then it follows that 
    \[
    w^\intercal(Q^\intercal(X+t Z)Q)w<0,
    \]
    which is a contradiction. 
    This implies that there exists $t^\star\ne 0$ such that 
    \[
    {\mu}_{1}(Q^\intercal(X+t^* Z)Q)=0,
    \]
     Hence we have 
    \[
    r>\mathrm{rank}(Q^\intercal(X+t^* Z)Q)=\mathrm{rank}(X+t^*Z)
    \]
    and $Q^\intercal(X+t^* Z)Q$ is positive semi-definite. To show that $X+t^* Z$ is positive semi-definite, take any $x\in\mathbb{R}^{n}$ and consider the decomposition $x=Qy+z$ where $y\in\mathbb{R}^r$ and $z\in\mathcal{N}(Q)=\mathcal{N}(X)\subseteq\mathcal{N}(Z)$. Then, we have
    \begin{align*}
        y^\intercal(X+t^\star Z)y &= (y^\intercal Q^\intercal+z^\intercal)(X+t^*Z)(Qy+z)\\
        &=y^\intercal Q^\intercal(X+t^*Z)Qy\geq 0. 
    \end{align*}
\end{proof}

Finally, the following lemma and its proof are similar to the previous one, but we state it separately for the sake of clarity. It will be used in the proof of Theorem~\ref{thm:existence}. 

\begin{lemma}\label{lem:reduction2}
     Suppose $X\in\mathbb{S}_{+}^{n}$ which is nonzero and let $Z\in\mathbb{S}^{n}$ be a nonzero symmetric matrix such that $\mathcal{R}(Z)\subseteq \mathcal{R}(X)$. Then there exists $t^*>0$ such that $X\pm t^*Z$ is positive semi-definite.
\end{lemma}
\begin{proof}
Let $\mathrm{rank}(X)=r$ and $\{y_1,\dots,y_{r}\}$ be orthonormal eigenvectors of nonzero eigenvalues of $X$. Since $y_{i}^\intercal Xy_{i}>0$ for all $y_i$, $i=1,\cdots,r$, there exists an interval $(-a_i,a_i)$ for $a_i>0$ such that $y_i^\intercal(X\pm tZ)y_i\geq 0$ for $t\in(-a_i,a_i)$. Take $t^*=\min\{a_1,\dots, a_r\}$. Then $t^*$ satisfies the statement of the theorem.
\end{proof}
Now we provide the complete proof of Theorem~\ref{thm:existence}. 
\begin{proof}[Proof of Theorem~\ref{thm:existence}]
Suppose $Z^\star_{\lambda}\in\mathbb{S}_{+}^{(m+n)}$ is a global minimizer of
\[
F(Z)= \hat{L}(\bar{Z})+\frac{\lambda}{2}\mathbf{tr}({Z})=\frac{1}{N}
\sum^N_{i=1}
\ell\left(
f^{}_{\mathbf{W}_0}(X_i)+\langle \mathbf{G}^{}(X_i), \bar{Z}\rangle 
 , Y_i\right)+\frac{\lambda}{2}\mathbf{tr}({Z})
\]
which is induced from $\boldsymbol{\delta}_{\lambda}^\star\in\mathbb{R}^{m\times n}$ by Lemma~\ref{lem:PQequiv}. Suppose there exists nonzero symmetric matrix $Z$ such that $Z\in\mathcal{S}(Z^\star_{\lambda})\triangleq\{Z\in\mathbb{S}_{}^{(m+n)}: \mathcal{R}(Z)\subseteq\mathcal{R}(Z^\star_{\lambda})\}$ and $\langle \mathbf{G}(X_i), Z \rangle = \mathbf{0}$ for $1\leq i \leq N$. In other words, $Z\in\mathcal{S}(Z^\star_{\lambda})\cap \mathcal{N}(\mathcal{A})$ where $\mathcal{A}\colon \mathbb{S}^{(m+n)}\rightarrow\mathbb{R}^{KN}$ is a linear operator defined as 
\[
\mathcal{A}(Z)_{ij}= \langle \mathbf{G}^{(j)}(X_i), \bar{Z} \rangle, \qquad 1\leq i \leq N,\quad  1\leq j \leq K. 
\]
Then there exists $t>0$ such that $Z^\star_{\lambda}\pm tZ$ is positive semi-definite by Lemma~\ref{lem:reduction2}, since $Z^\star$ must be nonzero. Therefore $\mathbf{tr}(Z)=0$, otherwise it will contradict the minimality of $Z^\star_{\lambda}$. 
Also we know that there exists nonzero $t^*\in\mathbb{R}$ such that $Z^\star_{\lambda}+t^*Z$ is also positive semi-definite with strictly lower rank by Lemma~\ref{lem:reduction}. Since $\mathbf{tr}(Z)=0$, $Z^\star_{\lambda}+t^*Z$ is also a global minimizer of $F$.  Replace $Z^\star_{\lambda}$ with $Z^\star_{\lambda}+tZ$ and repeat this process until we find a solution $Z^\star_{\lambda}$ with
\[
\{\mathbf{0}\} = \mathcal{S}(Z^\star_{\lambda}) \cap \mathcal{N}(\mathcal{A}).
\]
Now we let $\mathrm{rank}(Z^\star_{\lambda})=r$.
Then by dimension counting, we have the following inequality.
\begin{align*}
    0&=\mathrm{dim}\mathcal{S}(Z^\star_{\lambda})+\mathrm{dim}\mathcal{N}(\mathcal{A})-\mathrm{dim}(\mathcal{S}(Z^\star_{\lambda})+\mathcal{N}(\mathcal{A}))\\
    &=\mathrm{dim}\mathcal{S}(Z^\star_{\lambda})+\mathrm{dim}(\mathbb{S}^{(m+n)})-\mathrm{dim}\mathcal{R}
(\mathcal{A})-\mathrm{dim}(\mathcal{S}(Z^\star_{\lambda})+\mathcal{N}(\mathcal{A}))\\
 &= \mathrm{dim}\mathcal{S}(Z^\star_{\lambda})-KN+\mathrm{dim}(\mathbb{S}^{(m+n)})-\mathrm{dim}(\mathcal{S}(Z^\star_{\lambda})+\mathcal{N}(\mathcal{A}))\\
&=\mathrm{dim}\mathcal{S}(Z^\star)-KN +\mathrm{dim}(\mathcal{S}(Z^\star)^{\perp}\cap\mathcal{R}(\mathcal{A}))\\
&\geq\mathrm{dim}\mathcal{S}(Z^\star_{\lambda})-KN
\end{align*}
Now we prove that $\mathrm{dim}\mathcal{S}(Z^\star_{\lambda})=\frac{r(r+1)}{2}$ to complete the proof. Consider the diagonalization $Z^\star_{\lambda}=U\Lambda U^\intercal$ where $U$ is a orthogonal matrix. Since the dimension of the subspace is invariant under orthogonal transformations, we have
\[
\mathrm{dim}\mathcal{S}(Z^\star_{\lambda}) = \mathrm{dim}\mathcal{S}(\Lambda) = \mathrm{dim}\{Z\in\mathbb{S}_{}^{(m+n)}: \mathcal{R}(Z)\subseteq\mathcal{R}(\Lambda)\}
\]
where $\Lambda$ is diagonal matrix with nontrivial entries in the leading principle minor of size $r\times r$. This restricts the symmetric matrix $Z$ to have nontrivial entries only in the leading $r\times r$ block. Hence, $\mathrm{dim}\mathcal{S}(Z^\star_{\lambda})=\frac{r(r+1)}{2}.$
\end{proof}

\section{Omitted proof of Lemma~\ref{lem:sardthm}}\label{a:sardthm}
We prove Lemma~\ref{lem:sardthm} in this section. 
\begin{proof}[Proof of Lemma~\ref{lem:sardthm}]
Let 
$\Pi_{V^\perp}\colon\mathbb{R}^d\rightarrow V^\perp$ be the orthogonal projection onto the orthogonal complement of $V$ in $\mathbb{R}^{d}$. Then, $\Pi_{V^\perp}|_\mathcal{M}\colon\mathcal{M}\rightarrow V^\perp$ is a smooth mapping between manifolds.
Since
\[
\mathrm{dim}V^\perp= d-n > m = \mathrm{dim}\mathcal{M},
\]
$p$ is singular for all $p\in\mathcal{M}$. Therefore $\Pi_{V^\perp}(\mathcal{M})$ has measure zero in $\mathbb{R}^{d-n}$ by Sard's theorem.
Note that $\mathcal{M}+V \subseteq \Pi_{V^\perp}(\mathcal{M}) + V$
and the measure of $\Pi_{V^\perp}(\mathcal{M})+V$ in $\mathbb{R}^{d}$ is zero.
This concludes that $\mathcal{M}+V$ is measure-zero in $\mathbb{R}^{d}$.
\end{proof}

As a remark, the prior works of \cite{boumal2016non,du2018power} also use dimension-counting arguments that would warrant the use of Lemma~\ref{lem:sardthm}, but they do not provide a precise justification. Our Theorem~\ref{thm:second} makes a similar argument, but does so fully rigorous through Lemma~\ref{lem:sardthm}. 

\section{Generalization guarantee}\label{a:gen}
In this section, let $\ell(\cdot, \cdot)$ be our loss function which is convex, non-negative, and twice-differentiable on the first argument.
Then, our empirical risk is
\[
\hat{L}(\boldsymbol{\delta})=\frac{1}{N}
\sum^N_{i=1}
\ell\left(
f^{}_{\mathbf{W}_0}(X_i)+\langle \mathbf{G}^{}(X_i), \boldsymbol{\delta}\rangle 
, Y_i\right).
\]
We start the analysis from this non-regularized risk and expand it to regularized ones. We assume that our model is class of affine predictors $X\mapsto f^{}_{\mathbf{W}_0}(X)+\langle \mathbf{G}(X),\boldsymbol{\delta}\rangle$ for given data $X$. Now we apply the theory of Rademacher complexity to derive the upper bound of the generalization bound. To begin with, we start with introducing the classical result in probability theory from \citep{mcdiarmid1989method} without proof.
\begin{lemma}\label{lem:rada1}(McDiarmid inequality)
    Let \( X_1, \ldots, X_N \in \mathcal{X}\) be i.i.d $N$ random samples from dataset $\mathcal{X}$. Let \( g: \mathcal{X}^N \rightarrow \mathbb{R} \) be a function satisfying the following property with \( c > 0 \):

\[
\left| g(X_1, \ldots, X_{i-1}, X_i, X_{i+1}, \ldots, X_N) - g(X_1, \ldots, X_{i-1}, X'_i, X_{i+1}, \ldots, X_N) \right| \leq c
\]

for all  \( X_1, \ldots, X_N, X'_i \in\mathcal{X} \). Then, for all \( \varepsilon > 0 \),

\[
\mathbb{P} \left( \left| g(X_1, \ldots, X_N) - \mathbb{E}[g(X_1, \ldots, X_N)] \right| \geq \varepsilon \right) \leq \exp \left( -\frac{2 \varepsilon^2}{N c^2} \right).
\]
\end{lemma}
 Now, we define the \emph{Rademacher complexity} of the class of functions \(\mathcal{H}\) from \(\mathcal{X}\) to \(\mathbb{R}\):
\[
R_N (\mathcal{H}) = \mathbb{E}_{\varepsilon,\mathcal{D}} \left( \sup_{h \in \mathcal{H}} \frac{1}{N} \sum_{i=1}^{N} \varepsilon_i h(X_i) \right),
\]
where  $\{\varepsilon_i\}_{1\leq i \leq N}$ are independent Rademacher random variables, and $\mathcal{D}=\{X_1,\dots, X_N\}$ is $N$ random samples from $\mathcal{X}$. In our analysis, we will focus on class of affine predictors  $X_i \mapsto f^{}_{\mathbf{W}_0}(X_i)+\langle \mathbf{G}(X_i),\boldsymbol{\delta}\rangle$ and composition of affine predictors with loss $X_i\mapsto\ell( f^{}_{\mathbf{W}_0}(X_i)+\langle \mathbf{G}(X_i),\boldsymbol{\delta}\rangle, Y_i)$. Rademacher complexities are closely related to upper bounds on generalization bound due to the following lemma.
\begin{lemma}\label{lem:rada2}
Let $R_{N} (\mathcal{H})$ be the Rademacher complexity of the class of functions \(\mathcal{H}\) from \(\mathcal{X}\) to \(\mathbb{R}\) and $X_1,\dots, X_N$ are $N$ samples from $\mathcal{X}$. Then the following inequality holds.
\[
\mathbb{E} \left[ \sup_{h \in \mathcal{H}} \left( \frac{1}{N} \sum_{i=1}^{N} h(X_i) - \mathbb{E}[h(X)] \right) \right] \leq 2R_N (\mathcal{H}), \quad \mathbb{E} \left[ \sup_{h \in \mathcal{H}} \left( \mathbb{E}[h(X)] - \frac{1}{N} \sum_{i=1}^{N} h(X_i) \right) \right] \leq 2R_N (\mathcal{H}).
\]
\end{lemma}
\begin{proof}
    The proof is by using standard symmetrization arguments. We defer its proof to Theorem 8 of \citep{bartlett2002Rademacher}, or Section 4.5 of \citep{bach2021learning}.
\end{proof}
The next lemma uses a contraction property to reduce the Rademacher complexity of losses to linear predictors. These type of results are widely used in Rademacher analysis and we use the following specific version of contraction, which was originally introduced in Corollary 4 of \citep{maurer2016vector} and adapted to our setting. Write $\|\cdot\|_2$ for Euclidean vector norm.

\begin{lemma}\label{lem:rada3}
Let $\mathcal{A}$ be the class of functions \( a : \mathcal{X} \rightarrow \mathbb{R}^{K} \). For $1\leq i\leq N$, let $\ell_i\colon \mathbb{R}^{K}\rightarrow\mathbb{R}$ be G-Lipschitz continuous on $\mathcal{A}$ with respect to the Euclidean norm in the sense that the following holds:
\[
|\ell_i(a(X_1))-\ell_i(a'(X_2)) |\leq G \|a(X_1)-a'(X_2)\|_{2} \qquad \textrm{for any} \, \, a,a'\in\mathcal{A}, \quad X_1 ,X_2\in\mathcal{X}.
\]
Then we have the following inequality for independent Rademacher random variables $\{\sigma_{i}\}_{1\leq i \leq N}$ and $\{\varepsilon_{ij}\}_{1\leq i \leq N, 1\leq j \leq K}$:
\[
\mathbb{E}_{\sigma, \mathcal{D}} \left[ \sup_{a\in\mathcal{A}} \frac{1}{N} \sum_{i=1}^{N} \sigma_i \ell_i(a(X_i))  \right] \leq \sqrt{2}G\cdot \mathbb{E}_{\varepsilon, \mathcal{D}} \left[ \sup_{a\in\mathcal{A}} \frac{1}{N} \sum_{i=1}^{N} \sum_{j=1}^{K}\varepsilon_{ij} a_{j}(X_i) \right],
\]
where $a_j$ denotes the $j$-th coordinate of $a$ and $\mathcal{D}=\{(X_i,Y_i)\}_{i\in\{1,\dots,N\}}$ are i.i.d $N$ random samples sampled from $\mathcal{X}$.
\end{lemma}
\begin{proof}
   We defer the proof to the Section 5 of \citep{maurer2016vector}.
\end{proof}
In Lemma~\ref{lem:rada3}, if we sample $\mathcal{D}$ from a probability distribution $\mathcal{P}$, we can relax the Lipschitz continuity condition to hold for $\mathcal{P}$- almost surely. In other words, 
\[
|\ell(a(X_1))-\ell(a'(X_2)) |\leq G \|a(X_1)-a'(X_2)\|_{2} \qquad \textrm{for any} \, \, a,a'\in\mathcal{A}, \quad X_1 ,X_2\subseteq\mathcal{D}\sim \mathcal{P}.
\]
The next lemma states that the Rademacher complexity of class of bounded affine predictors decays at most $\mathcal{O}(\frac{1}{\sqrt{N}})$ rate.
\begin{lemma}\label{lem:rada4}
Assume $ \mathcal{D}=\{(X_i,Y_i)\}_{i\in\{1,\dots,N\}}$ is i.i.d $N$ random samples sampled from probability distribution $\mathcal{P}$. Assume $\mathcal{A}_{D}=\{X_i\mapsto f_{\mathbf{W}_{0}}(X_i)+\langle \mathbf{G}(X_i),\boldsymbol{\delta}\rangle\in\mathbb{R}^{K}: \|\boldsymbol{\delta}\|_{*}\leq D , \boldsymbol{\delta}\in\mathbb{R}^{m\times n} \}$ is class of affine predictors with bounded nuclear norm $D>0$. Suppose $\|\mathbf{G}^{(j)}(X_i)\|_{F}\leq R$ almost surely with respect to the random data $X_i\sim\mathcal{P}$. Then, 
\[
\mathbb{E}_{\varepsilon, \mathcal{D}} \left[ \sup_{a\in\mathcal{A}} \frac{1}{N} \sum_{i=1}^{N} \sum_{j=1}^{K}\varepsilon_{ij} a_{j}(X_i) \right]\leq \frac{RD\sqrt{K}}{\sqrt{N}}
\]
where $\{\varepsilon_{ij}\}_{1\leq i \leq N, 1\leq j \leq K}$ are i.i.d Rademacher random variables.
\end{lemma}
\begin{proof}
\begin{align*}
  \mathbb{E}_{\varepsilon} \left[ \sup_{a\in\mathcal{A}} \frac{1}{N} \sum_{i=1}^{N} \sum_{j=1}^{K}\varepsilon_{ij} a_{j}(X_i)\right]   &=\mathbb{E}_{\varepsilon} \left[ \sup_{\|\boldsymbol{\delta}\|_{*}\leq D} \frac{1}{N} \sum_{i=1}^{N} \sum_{j=1}^{K} \varepsilon_{ij}\left(f^{(j)}_{\mathbf{W_0}}(X_i)+ \langle \mathbf{G}^{(j)}(X_i),\boldsymbol{\delta}\rangle \right)\right]\\
  &=\mathbb{E}_{\varepsilon} \left[ \sup_{\|\boldsymbol{\delta}\|_{*}\leq D} \frac{1}{N} \sum_{i=1}^{N} \sum_{j=1}^{K} \varepsilon_{ij}\langle \mathbf{G}^{(j)}(X_i),\boldsymbol{\delta}\rangle +\frac{1}{N}\sum_{i=1}^{N}\sum_{j=1}^{K}\varepsilon_{ij}f^{(j)}_{\mathbf{W_0}}(X_i) \right]\\
   &=\mathbb{E}_{\varepsilon} \left[ \sup_{\|\boldsymbol{\delta}\|_{*}\leq D} \frac{1}{N} \sum_{i=1}^{N} \sum_{j=1}^{K} \varepsilon_{ij}\langle \mathbf{G}^{(j)}(X_i),\boldsymbol{\delta}\rangle\right] \\
      &\leq \mathbb{E}_{\varepsilon} \left[ \sup_{\|\boldsymbol{\delta}\|_{F}\leq D} \frac{1}{N} \sum_{i=1}^{N} \sum_{j=1}^{K} \varepsilon_{ij}\langle \mathbf{G}^{(j)}(X_i),\boldsymbol{\delta}\rangle\right] \\
  &=\mathbb{E}_{\varepsilon} \left[ \frac{D}{N}\sup_{\|\boldsymbol{\delta}\|_{F}\leq 1}  \sum_{i=1}^{N}\sum_{j=1}^{K}  \varepsilon_{ij} \langle \mathbf{G}^{(j)}(X_i),\boldsymbol{\delta}\rangle\right]\\
    &=\frac{D}{N}\mathbb{E}_{\varepsilon} \left\| \sum_{i=1}^{N} \sum_{j=1}^{K} \varepsilon_{ij} \mathbf{G}^{(j)}(X_i)\right\|_{F}.
\end{align*}
The inequality is from the fact that $\|\cdot\|_{F}\leq \|\cdot\|_{*}$, hence $\{\boldsymbol{\delta} : \|\boldsymbol{\delta}\|_{*} \leq D\}\subset \{\boldsymbol{\delta} : \|\boldsymbol{\delta}\|_{F} \leq D\}$. The last equality is from the fact that $\|\cdot\|_F$ is self-dual. Next, we can bound $\mathbb{E}_{\varepsilon}\left\|\sum_{i=1}^{N} \sum_{j=1}^{K} \varepsilon_{ij} \mathbf{G}^{(j)}(X_i)\right\|_{F}$ by the following inequalities.
\begin{align*}
    \mathbb{E}_{\varepsilon}\left\| \sum_{i=1}^{N} \sum_{j=1}^{K} \varepsilon_{ij} \mathbf{G}^{(j)}(X_i)\right\|_{F}&\leq\sqrt{\mathbb{E}_{\varepsilon}\left\| \sum_{i=1}^{N} \sum_{j=1}^{K} \varepsilon_{ij} \mathbf{G}^{(j)}(X_i)\right\|_{F}^2}\\
    &=\sqrt{\mathbb{E}_{\varepsilon}\sum_{i=1}^{N} \sum_{j=1}^{K}\left\|  \varepsilon_{ij} \mathbf{G}^{(j)}(X_i)\right\|_{F}^2}\\
     &=\sqrt{\sum_{i=1}^{N} \sum_{j=1}^{K}\left\|  \mathbf{G}^{(j)}(X_i)\right\|_{F}^2}\\
     &\leq R\sqrt{NK}. \quad \mathrm{a.s.}
\end{align*}
The first inequality is from Jensen's inequality, the equalities are from i.i.d assumption of $\varepsilon_{ik}$. We combine the results and take expectation with respect to $\mathcal{D}$ to get 
\[
\mathbb{E}_{\varepsilon, \mathcal{D}} \left[ \sup_{a\in\mathcal{A}} \frac{1}{N} \sum_{i=1}^{N} \sum_{j=1}^{K}\varepsilon_{ij} a_{k}(X_i) \right]\leq \frac{D}{N}\cdot R\sqrt{NK}=\frac{RD\sqrt{K}}{\sqrt{N}}.
\]
\end{proof}
We then combine the previous results to get the following Lemma. 
\begin{lemma}\label{lem:rada5}
Assume $ \mathcal{D}=\{(X_i,Y_i)\}_{i\in\{1,\dots,N\}}$ is i.i.d $N$ random samples sampled from probability distribution $\mathcal{P}$. Let $\hat{L}$ is non-regularized empirical risk defined as 
\[
\hat{L}(\boldsymbol{\delta})=\frac{1}{N}
\sum^N_{i=1}
\ell\left(
f^{}_{\mathbf{W}_0}(X_i)+\langle \mathbf{G}^{}(X_i), \boldsymbol{\delta}\rangle 
, Y_i\right)
\]
and $\mathcal{A}_{D}=\{X_i\mapsto f_{\mathbf{W}_{0}}(X_i)+\langle \mathbf{G}(X_i),\boldsymbol{\delta}\rangle\in\mathbb{R}^{K}: \|\boldsymbol{\delta}\|_{*}\leq D , \boldsymbol{\delta}\in\mathbb{R}^{m\times n} \}$ is class of affine predictors with bounded nuclear norm $D$.
For $1\leq j\leq K$, suppose $\|\mathbf{G}^{(j)}(X)\|_{F}\leq R$ almost surely with respect to the random data $X_i\sim\mathcal{P}$. For $1\leq i\leq N$, suppose $\ell_i\triangleq\ell(\cdot, Y_i)$ is $G$-Lipschitz continuous on $\mathcal{A}$ on the first argument (with respect to the Euclidean norm) for almost surely with respect to the random data $X_i\subseteq\mathcal{D}\sim\mathcal{P}$. That is, 
\[
|\ell_i(a(X_1))-\ell_i(a'(X_2)) |\leq G \|a(X_1)-a'(X_2)\|_{2} \qquad \textrm{for any} \, \, a,a'\in\mathcal{A}, \quad X_1 ,X_2\subseteq\mathcal{D}\sim \mathcal{P}.
\]
Then for any $\|\boldsymbol{\delta}\|_{*}\leq D$, fixed $\boldsymbol{\delta}_0$ such that $\|\boldsymbol{\delta}_0\|_{*}\leq D$, and $\eta\in(0,1)$, the following inequality holds with probability greater than $1-\eta$:
\[
\hat{L}(\boldsymbol{\delta}_0)-\hat{L}(\boldsymbol{\delta})-L(\boldsymbol{\delta}_0)+{L}(\boldsymbol{\delta})<\frac{\sqrt{2K}GRD}{\sqrt{N}}\left(2+\sqrt{\log{\frac{1}{\eta}}}\right).
\]
\end{lemma}
\begin{proof}
Take $g$ of Lemma~\ref{lem:rada1} to be $g=\sup_{ \|\boldsymbol{\delta}\|_{*}\leq D}(\hat{L}(\boldsymbol{\delta}_0)-\hat{L}(\boldsymbol{\delta})-L(\boldsymbol{\delta}_0)+{L}(\boldsymbol{\delta}))$, which is a function of $X_1,\dots,X_N$. Since $\|\boldsymbol{\delta}\|_{*}\leq D$ implies $\|\boldsymbol{\delta}\|_{F}\leq D$ and by the Lipschitz continuity of $\ell(\cdot, Y_i)$, we have the following for any $(X_i,Y_i)\in\mathcal{D}$:
\begin{align*}
    |\ell\left(
f^{}_{\mathbf{W}_0}(X_i)+\langle \mathbf{G}^{}(X_i), \boldsymbol{\delta}_0\rangle 
, Y_i\right)- \ell\left(
f^{}_{\mathbf{W}_0}(X_i)+\langle \mathbf{G}^{}(X_i), \boldsymbol{\delta}\rangle 
, Y_i\right)|&\leq G\|\langle \boldsymbol{\delta}_0-\boldsymbol{\delta}, \mathbf{G}(X_i)\rangle \|_2\\
    &\leq G\sqrt{\sum_{j=1}^{K}\|\boldsymbol{\delta}_0-\boldsymbol{\delta}\|_F^2\|\mathbf{G}^{(j)}(X_i)\|_F^2}\\
     &\leq G\sqrt{\sum_{j=1}^{K}\|\boldsymbol{\delta}_0-\boldsymbol{\delta}\|_{*}^2\|\mathbf{G}^{(j)}(X_i)\|_F^2}\\
         &\leq G\sqrt{\sum_{j=1}^{K}4D^2\cdot R^2}\\
         &=2GRD\sqrt{K}.
\end{align*}Hence if we change only one data point $(X_i,Y_i)$ of $g$ to $(X_i^{'},Y_i^{'})$, the deviation of $\hat{L}(\boldsymbol{\delta}_0)-\hat{L}(\boldsymbol{\delta})$ is at most $\frac{2GRD\sqrt{K}}{N}$. Then by Lemma~\ref{lem:rada1}, we have
\[
\sup_{ \|\boldsymbol{\delta}\|_{*}\leq D}(\hat{L}(\boldsymbol{\delta}_0)-\hat{L}(\boldsymbol{\delta})-L(\boldsymbol{\delta}_0)+{L}(\boldsymbol{\delta}))<\mathbb{E}\left[\sup_{ \|\boldsymbol{\delta}\|_{*}\leq D}(\hat{L}(\boldsymbol{\delta}_0)-\hat{L}(\boldsymbol{\delta})-L(\boldsymbol{\delta}_0)+{L}(\boldsymbol{\delta}))\right]+\frac{t\sqrt{2K}GRD}{\sqrt{N}}
\]
with probability greater than $1-e^{-t^2}$.  The expectation on the right hand side can be reduced to
\begin{align*}
\mathbb{E}_{\mathcal{D}}\left[\sup_{ \|\boldsymbol{\delta}\|_{*}\leq D}(\hat{L}(\boldsymbol{\delta}_0)-\hat{L}(\boldsymbol{\delta})-L(\boldsymbol{\delta}_0)+{L}(\boldsymbol{\delta}))\right]&=\mathbb{E}_{\mathcal{D}}\left[\sup_{ \|\boldsymbol{\delta}\|_{*}\leq D}(-\hat{L}(\boldsymbol{\delta})+{L}(\boldsymbol{\delta}))+\hat{L}(\boldsymbol{\delta}_0)-L(\boldsymbol{\delta}_0)\right]\\
&=\mathbb{E}_{\mathcal{D}}\left[\sup_{ \|\boldsymbol{\delta}\|_{*}\leq D}({L}(\boldsymbol{\delta})-\hat{L}(\boldsymbol{\delta}))\right]
\end{align*}
Note that
\begin{align*}L(\boldsymbol{\delta})-\hat{L}(\boldsymbol{\delta})&={L}(\boldsymbol{\delta})-\frac{1}{N}\sum_{i=1}^{N}\ell(f^{}_{\mathbf{W}_0}(X_i)+\langle \mathbf{G}^{}(X_i), \boldsymbol{\delta}\rangle , Y_i)\\
 &=\mathbb{E}\Big[\ell(f^{}_{\mathbf{W}_0}(X_i)+\langle \mathbf{G}^{}(X_i), \boldsymbol{\delta}\rangle , Y_i)\Big]-\frac{1}{N}\sum_{i=1}^{N}\ell(f^{}_{\mathbf{W}_0}(X_i)+\langle \mathbf{G}^{}(X_i), \boldsymbol{\delta}\rangle , Y_i),
\end{align*}
where the expectation is taken over $X_i\sim\mathcal{P}$.
Now apply Lemma~\ref{lem:rada2} to get
\begin{align*}
\mathbb{E}_{\mathcal{D}}\left[\sup_{ \|\boldsymbol{\delta}\|_{*}\leq D}({L}(\boldsymbol{\delta})-\hat{L}(\boldsymbol{\delta}))\right] &= \mathbb{E}_{\mathcal{D}}\left[\sup_{ \|\boldsymbol{\delta}\|_{*}\leq D}(\mathbb{E}\Big[\ell(f^{}_{\mathbf{W}_0}(X_i)+\langle \mathbf{G}^{}(X_i), \boldsymbol{\delta}\rangle , Y_i)\Big]\right]\\
&\qquad -\mathbb{E}_{\mathcal{D}}\left[\frac{1}{N}\sum_{i=1}^{N}\ell(f^{}_{\mathbf{W}_0}(X_i)+\langle \mathbf{G}^{}(X_i), \boldsymbol{\delta}\rangle , Y_i))\right]\\
&\leq 2\mathbb{E}_{\sigma, \mathcal{D}} \left[ \sup_{\|\boldsymbol{\delta}\|_{*}\leq D} \frac{1}{N} \sum_{i=1}^{N} \sigma_{i}\ell(f^{}_{\mathbf{W}_0}(X_i)+\langle \mathbf{G}^{}(X_i), \boldsymbol{\delta}\rangle , Y_i))\right]
\end{align*}
where $\{\sigma\}_{1\leq i \leq N}$ are i.i.d Rademacher variables. Then apply Lemma~\ref{lem:rada3} to get
\begin{align*}
&2\mathbb{E}_{\sigma, \mathcal{D}} \left[ \sup_{\|\boldsymbol{\delta}\|_{*}\leq D} \frac{1}{N} \sum_{i=1}^{N} \sigma_{i}\ell(f^{}_{\mathbf{W}_0}(X_i)+\langle \mathbf{G}^{}(X_i), \boldsymbol{\delta}\rangle , Y_i))\right] \\
&= 2\sqrt{2}G\mathbb{E}_{\varepsilon, \mathcal{D}} \left[ \sup_{\|\boldsymbol{\delta}\|_{*}\leq D} \frac{1}{N} \sum_{i=1}^{N}\sum_{j=1}^{K}  \varepsilon_{ij}\left(f^{j}_{\mathbf{W}_0}(X_i)+\langle \mathbf{G}^{j}(X_i), \boldsymbol{\delta}\rangle\right)\right]
\end{align*}
where $\{\varepsilon_{ij}\}_{1\leq i \leq N, 1\leq j \leq K}$ are i.i.d Rademacher random variables. Finally, use Lemma~\ref{lem:rada4} to get 
\[
 2\sqrt{2}G\mathbb{E}_{\varepsilon, \mathcal{D}} \left[ \sup_{\|\boldsymbol{\delta}\|_{*}\leq D} \frac{1}{N} \sum_{i=1}^{N}\sum_{j=1}^{K}  \varepsilon_{ij}\left(f^{j}_{\mathbf{W}_0}(X_i)+\langle \mathbf{G}^{j}(X_i), \boldsymbol{\delta}\rangle\right)\right]\leq 2\sqrt{2}G\cdot \frac{RD\sqrt{K}}{\sqrt{N}}.
\]
Therefore, we conclude that 
\[
\hat{L}(\boldsymbol{\delta}_0)-\hat{L}(\boldsymbol{\delta})-L(\boldsymbol{\delta}_0)+{L}(\boldsymbol{\delta})< \frac{\sqrt{2K}GRD}{\sqrt{N}}\left(2+t\right).
\]
for $\|\boldsymbol{\delta}\|_{*}\leq D$ with probability greater than $1-e^{-t^2}$. By reparametrization, we get 
\[
\hat{L}(\boldsymbol{\delta}_0)-\hat{L}(\boldsymbol{\delta})-L(\boldsymbol{\delta}_0)+{L}(\boldsymbol{\delta})< \frac{\sqrt{2K}GRD}{\sqrt{N}}\left(2+\sqrt{\log{\frac{1}{\eta}}}\right).
\]
for $\|\boldsymbol{\delta}\|_{*}\leq D$ with probability greater than $1-\eta$. 
\end{proof}
Now we can extend this generalization guarantee of constrained optimization to regularized optimization, which aligns with our problem of interest. For notational convenience, let
\[
L_{\lambda}(\boldsymbol{\delta})= L(\boldsymbol{\delta}) + \lambda \|\boldsymbol{\delta}\|_{*} , \quad \hat{L}_{\lambda}(\boldsymbol{\delta})= \hat{L}(\boldsymbol{\delta}) + \lambda \|\boldsymbol{\delta}\|_{*} 
\]
We follow the proof structure of \citep{bach2021learning}, which was motivated by \citep{bartlett2005local} and \citep{sridharan2008fast}.
\begin{theorem}\label{thm:gen}
  Fix $\varepsilon>0$ and let $0\neq\boldsymbol{\delta}^\star_\mathrm{true} \in \argmin_{\boldsymbol{\delta}}L(\boldsymbol{\delta})$ be the true optimum of the population risk and consider the setup of Lemma~\ref{lem:rada5} with $D=(2+\varepsilon)\|\boldsymbol{\delta}^\star_\mathrm{true}\|_{*}$, which is the upper bound on the nuclear norm of the predictors. Let $\eta\in(0,1)$ and
\[
   \lambda=\frac{(2+\varepsilon)\sqrt{2K}GR}{\sqrt{N}}\left(2+\sqrt{\log{\frac{1}{\eta}}}\right).
   \]    Write ${\boldsymbol{\delta}}^\star_\lambda$ to denote a minimizer (not necessarily unique) of $\hat{L}_{\lambda}(\boldsymbol{\delta})$.Consider the setup of Corollary~\ref{cor:opt} with $P$ randomly sampled with a probability distribution supported in
\[
   \Big\{P\in\mathbb{S}_{+}^{(m+n)}: \|P\|_{F}<  \frac{\varepsilon{\lambda}\|\boldsymbol{\delta}^\star_{\mathrm{true}}\|_{*}}{2\|\boldsymbol{\delta}^\star_{\lambda}\|_{*}}\Big\}
\]
and is absolutely continuous with respect to the Lebesgue measure on $\mathbb{S}^{(m+n)}\cong\mathbb{R}^{\frac{(m+n)(m+n+1)}{2}}$.
   Let $(\hat{\mathbf{u}},\hat{\mathbf{v}})$ be an SOSP of $\hat{L}_{\lambda,P}$.  Then with probability greater than $1-\eta$, 
\[
\!\!\!
L(\hat{\mathbf{u}}\hat{\mathbf{v}}^\intercal) -L(\boldsymbol{\delta}^\star_{\mathrm{true}})<\|\boldsymbol{\delta}^\star_{\mathrm{true}}\|_{*}\frac{(2+\varepsilon)^2\sqrt{2K}GR}{\sqrt{N}}\left(2+\sqrt{\log{\frac{1}{\eta}}}\right).
\]
\end{theorem}
\begin{proof}
Let $\tilde{\varepsilon}=\frac{\varepsilon{\lambda}\|\boldsymbol{\delta}^\star_{\mathrm{true}}\|_{*}}{2\|\boldsymbol{\delta}^\star_{\lambda}\|_{*}}$ and consider the convex set 
\[
C=\Big\{\boldsymbol{\delta} : \|\boldsymbol{\delta}\|_{*}\leq 2\|\boldsymbol{\delta}^\star_\mathrm{true}\|_{*}+\frac{2\tilde{\varepsilon}}{\lambda}\|\boldsymbol{\delta}_{\lambda}^\star\|_{*}, L_{\lambda}(\boldsymbol{\delta})-L_{\lambda}(\boldsymbol{\delta}^\star_\mathrm{true})\leq  \lambda \|\boldsymbol{\delta}^\star_\mathrm{true}\|_{*} +2\tilde{\varepsilon}\|\boldsymbol{\delta}_{\lambda}^\star\|_{*}\Big\}.
\]
Then for $\|\boldsymbol{\delta}\|_{*}=2\|\boldsymbol{\delta}^\star_\mathrm{true}\|_{*}+\frac{2\tilde{\varepsilon}}{\lambda}\|\boldsymbol{\delta}_{\lambda}^\star\|_{*}$, $\boldsymbol{\delta}\notin \mathrm{int}C$ since the following inequalities hold.
\begin{align*}
    L_{\lambda}(\boldsymbol{\delta})-L_{\lambda}(\boldsymbol{\delta}^\star_\mathrm{true}) &= L(\boldsymbol{\delta})-L(\boldsymbol{\delta}^\star_\mathrm{true})+\lambda\|\boldsymbol{\delta}\|_{*}-\lambda\|\boldsymbol{\delta}^\star_\mathrm{true}\|_{*}\geq \lambda\|\boldsymbol{\delta}\|_{*}-\lambda\|\boldsymbol{\delta}^\star_\mathrm{true}\|_{*}=\lambda\|\boldsymbol{\delta}^\star_\mathrm{true}\|_{*}+2\tilde{\varepsilon}\|\boldsymbol{\delta}_{\lambda}^\star\|_{*}.
\end{align*}
Therefore the boundary $\partial C$ of $C$ should be
\[
\partial C=\Big\{\boldsymbol{\delta} : \|\boldsymbol{\delta}\|_{*}\leq 2\|\boldsymbol{\delta}^\star_\mathrm{true}\|_{*}+\frac{2\tilde{\varepsilon}}{\lambda}\|\boldsymbol{\delta}_{\lambda}^\star\|_{*}, L_{\lambda}(\boldsymbol{\delta})-L_{\lambda}(\boldsymbol{\delta}^\star_\mathrm{true})=  \lambda \|\boldsymbol{\delta}^\star_\mathrm{true}\|_{*}  +2\tilde{\varepsilon}\|\boldsymbol{\delta}_{\lambda}^\star\|_{*} \Big\}.
\]  
Now suppose $\hat{\mathbf{u}}\hat{\mathbf{v}}^\intercal \notin C$. Then since ${\boldsymbol{\delta}}^\star_{\mathrm{true}} \in C$, there exists 
 $\boldsymbol{\delta}$ in the segment $[\hat{\mathbf{u}}\hat{\mathbf{v}}^\intercal, {\boldsymbol{\delta}}^\star_{\mathrm{true}}]$ such that  $\boldsymbol{\delta}\in \partial C$. By the convexity of $\hat{L}_{\lambda}$, we have 
 \[
 \hat{L}_{\lambda}(\boldsymbol{\delta})\leq \max\left(\hat{L}_{\lambda}(\boldsymbol{\delta}^\star_{\mathrm{true}}), \hat{L}_{\lambda}(\hat{\mathbf{u}}\hat{\mathbf{v}}^\intercal)\right).
 \]
 Then we get
 \[
 \hat{L}_{\lambda}(\boldsymbol{\delta}^\star_{\mathrm{true}})-\hat{L}_{\lambda}(\boldsymbol{\delta})\geq -2\tilde{\varepsilon}\|\boldsymbol{\delta}_{\lambda}^\star\|_{*}
 \]
 by Corollary~\ref{cor:opt}.
 Therefore, 
 \begin{align}  
 \label{eq:thmbeq}
 \hat{L}(\boldsymbol{\delta}^\star_{\mathrm{true}})-\hat{L}(\boldsymbol{\delta})-L(\boldsymbol{\delta}^\star_{\mathrm{true}})+L(\boldsymbol{\delta}) &=  \hat{L}_{\lambda}(\boldsymbol{\delta}^\star_{\mathrm{true}})-\hat{L}_{\lambda}(\boldsymbol{\delta})-L_{\lambda}(\boldsymbol{\delta}^\star_{\mathrm{true}})+L_{\lambda}(\boldsymbol{\delta}) \nonumber \\ \
 &\geq L_{\lambda}(\boldsymbol{\delta})-L_{\lambda}(\boldsymbol{\delta}^\star)-2\tilde{\varepsilon}\|\boldsymbol{\delta}_{\lambda}^\star\|_{*} \\
 &=\lambda\|\boldsymbol{\delta}^{\star}_{\mathrm{true}}\|_{*} \nonumber
\end{align}
 Note that $\|\boldsymbol{\delta}\|_{*}\leq2\|\boldsymbol{\delta}^\star_\mathrm{true}\|_{*}+\frac{2\tilde{\varepsilon}}{\lambda}\|\boldsymbol{\delta}_{\lambda}^\star\|_{*}<(2+\varepsilon)\|\boldsymbol{\delta}^\star_\mathrm{true}\|_{*}$ and 
 \[
{\lambda}\|\boldsymbol{\delta}_{\mathrm{true}}^\star\|_{*}=\|\boldsymbol{\delta}^\star_\mathrm{true}\|_{*}\frac{(2+\varepsilon)\sqrt{2K}GR}{\sqrt{N}}\left(2+\sqrt{\log{\frac{1}{\eta}}}\right).
 \]
Then by Lemma~\ref{lem:rada5}, \eqref{eq:thmbeq} should happen with probability less than $\eta$. Then with probability greater than $1-\eta$, $\hat{\mathbf{u}}\hat{\mathbf{v}}^\intercal \in C$. In other words,
 \[
 L_{\lambda}(\hat{\mathbf{u}}\hat{\mathbf{v}}^\intercal )-L_{\lambda}(\boldsymbol{\delta}^\star_{\mathrm{true}})<\lambda \|{\boldsymbol{\delta}}^\star_{\mathrm{true}}\|_{*}+2\tilde{\varepsilon}\|\boldsymbol{\delta}^\star_{\lambda}\|_{*}.
 \]
Hence,
\begin{align*}
     L(\hat{\mathbf{u}}\hat{\mathbf{v}}^\intercal)+\lambda\|\hat{\mathbf{u}}\hat{\mathbf{v}}^\intercal\|_{*}&< L_{\lambda}(\boldsymbol{\delta}^\star_{\mathrm{true}})+\lambda \|{\boldsymbol{\delta}}^\star_{\mathrm{true}}\|_{*}+2\tilde{\varepsilon}\|\boldsymbol{\delta}^\star_{\lambda}\|_{*}\\
     &= L(\boldsymbol{\delta}^\star_{\mathrm{true}})+2\lambda \|{\boldsymbol{\delta}}^\star_{\mathrm{true}}\|_{*}+2\tilde{\varepsilon}\|\boldsymbol{\delta}^\star_{\lambda}\|_{*}\\
     &\leq L(\boldsymbol{\delta}^\star_{\mathrm{true}})+2\lambda\|\boldsymbol{\delta}^\star_{\mathrm{true}}\|_{*}+\varepsilon\lambda\|\boldsymbol{\delta}^\star_{\mathrm{true}}\|_{*}.
\end{align*}
Finally, we get
\[
 L(\hat{\mathbf{u}}\hat{\mathbf{v}}^\intercal) -L(\boldsymbol{\delta}^\star_{\mathrm{true}})< \|\boldsymbol{\delta}^\star_{\mathrm{true}}\|_{*}\frac{(2+\varepsilon)^2\sqrt{2K}GR}{\sqrt{N}}\left(2+\sqrt{\log{\frac{1}{\delta}}}\right).
\]
\end{proof}
By using the fact that $\ell^{CE}$ is Lipschitz continuous, we can reduce Theorem~\ref{thm:gen} to Theorem~\ref{thm:gen-main}. Note that the loss function $\ell$ may not be Lipschitz continuous in general. However, Lipschitz continuity is a mild assumption when the domain is restricted to a bounded class of predictors $\mathcal{A}_D$ of Lemma~\ref{lem:rada5}. 

\begin{proof}[Proof of Theorem~\ref{thm:gen-main}]
If $\ell(\cdot,Y)\colon \mathbb{R}^{K}\rightarrow\mathbb{R}$ is cross entropy loss defined as
 \[
    \ell(X,Y)=\ell^{CE}(X,Y)=-\log\left(\frac{\exp{X^{(j)}}}{\sum_{i=1}^{K}\exp{X^{(i)}}}\right) = -X^{(j)}+\log\left(\sum_{i=1}^{K}\exp{X^{(i)}}\right)
 \]
with true label $Y=j$, we have 
 \[
 \nabla \ell^{CE}(X,Y)_j = -1+\frac{\exp{X^{(j)}}}{\sum_{i=1}^{K}\exp{X^{(i)}}}= -\frac{\sum_{i\neq j}\exp{X^{(Y)}}}{\sum_{i=1}^{K}\exp{X^{(i)}}} 
 \]
 and for $k\ne j$,
  \[
 \nabla \ell^{CE}(X,Y)_k = \frac{\exp{X^{(k)}}}{\sum_{i=1}^{K}\exp{X^{(Y)}}}
 \]
 Then we can bound the Euclidean norm of the gradient as follows.
 \[
 \|\nabla \ell^{CE}(X,Y)\|_2^2=\frac{\left(\sum_{i\neq j}\exp{X^{(i)}}\right)^2}{\left(\sum_{i=1}^{K}\exp{X^{(i)}}\right)^2} +\frac{\sum_{i\ne j}\exp{2X^{(k)}}}{\left(\sum_{i=1}^{K}\exp{X^{(i)}}\right)^2}\leq 1+1=2.
 \]
 Hence the gradient of the cross entropy loss is bounded by $\sqrt{2}$ and we may replace $G$ in Theorem~\ref{thm:gen} with $\sqrt{2}$ to get 
\[
 L(\hat{\mathbf{u}}\hat{\mathbf{v}}^\intercal) -L(\boldsymbol{\delta}^\star_{\mathrm{true}})< \|\boldsymbol{\delta}^\star_{\mathrm{true}}\|_{*}\frac{2(2+\varepsilon)^2\sqrt{K}R}{\sqrt{N}}\left(2+\sqrt{\log{\frac{1}{\delta}}}\right).
\]
\end{proof}
\section{Details of experiments}\label{a:exp}

\paragraph{Optimizing nuclear norm.}
Recall that SGD or GD on the loss function with weight decay and with regularization parameter $\lambda$ is equivalent to minimizing 
\[
\frac{1}{N}
\sum^N_{i=1}
\ell\left(
f^{}_{\mathbf{W}_0}(X_i)+\langle \mathbf{G}^{}(X_i), \mathbf{u}\mathbf{v}^\intercal\rangle 
, Y_i\right)+\frac{\lambda}{2}\|\mathbf{u}\|_F^2+\frac{\lambda}{2}\|\mathbf{v}\|_F^2,
\]
with respect to $\mathbf{u}$ and $\mathbf{v}$.
In full fine-tuning however, this is equivalent to minimize the following with respect to $\boldsymbol{\delta}$:
\[
\frac{1}{N}
\sum^N_{i=1}
\ell\left(
f^{}_{\mathbf{W}_0}(X_i)+\langle \mathbf{G}^{}(X_i), \boldsymbol{\delta}\rangle 
, Y_i\right)+\lambda\|\boldsymbol{\delta}\|_{*}.
\]
The problem here is that gradient methods no longer apply since the nuclear norm is non-differentiable. Therefore, we use the proximal gradient method:
\[
\boldsymbol{\delta}_{t+1} = \mathbf{prox}_{\alpha\lambda \|\cdot\|_{*}}(\boldsymbol{\delta}_t - \alpha \nabla\hat{L}(\boldsymbol{\delta}_t))
\]
where
\[
\mathbf{prox}_{\alpha\lambda \|\cdot\|_{*}}(\boldsymbol{\delta}) = \argmin_{\boldsymbol{\delta}^\prime}\left(\lambda\|\boldsymbol{\delta}^{'}\|_{*}+\frac{1}{2\alpha}\|\boldsymbol{\delta}^\prime-\boldsymbol{\delta}\|_{F}^2\right).
\]
It is well known that the proximal gradient method on convex objective converges to a global minimum \cite{polyak1987introduction}.

\paragraph{Hyperparameters on NLP tasks}
For NLP tasks, we use full batch to perform GD on training. We only train the query ($W_q$) and value ($W_v$) weights of the RoBERTa-base model, which was empirically shown to have good performance \cite{hu2021lora}. Furthermore, calculating the proximal operator of a nuclear norm is a computational bottleneck during the training of all $W_q$ and $W_v$ matrices. Therefore, we limit our training to only the last layer of $W_q$ and $W_v$. To ensure a fair comparison, we apply the same approach to the LoRA updates. Additional information is in Table \ref{tab:hyperparam}.
\begin{table}[h]
\centering

\resizebox{0.78\textwidth}{!}{
\begin{tabular}{rcccccc}
\toprule
Task & \textbf{SST-2,QNLI} & \textbf{MR,CR,QQP,Subj} \\  
\midrule
 Batch size & 32 & 32 \\
Learning rate (Full, LoRA fine tuning)& 0.0005 & 0.001 \\
Trained layer & $W_{q} , W_{v} \ \textrm{(last layer only)} $ & $W_{q} , W_{v}$ \ \textrm{(last layer only)}\\
Weight decay  & 0.01 & 0.01 \\
\bottomrule
\end{tabular}
}
\caption{Hyperparameters on experiment in Section \ref{s:experiment} (NLP tasks)}
\label{tab:hyperparam}
\end{table}

\paragraph{Hyperparameters on image and speech classification tasks}
 Similar to NLP tasks, we train the last attention layers. Further details are in Table \ref{tab:hyperparam2}.
\begin{table}[h]
\centering

\resizebox{0.78\textwidth}{!}{
\begin{tabular}{rcccccc}
\toprule
Task & \textbf{Image classification} & \textbf{Speech classification} \\  
\midrule
 Batch size & 16 & 16 \\
Learning rate (Full, LoRA fine tuning)& 0.005 & 0.005 \\
Trained layer & $W_{q} , W_{v} \ \textrm{(last layer only)} $ & $W_{q} , W_{v}$ \ \textrm{(last layer only)}\\
Weight decay  & 0 & 0.001 \\
\bottomrule
\end{tabular}
}
\caption{Hyperparameters on experiment in Section \ref{s:experiment} (Image and speech classification tasks)}
\label{tab:hyperparam2}
\end{table}

\newpage

\paragraph{Test accuracy.}
For the setting of Section \ref{s:experiment} on NLP tasks, we additionally conduct evaluations on a test set of 1000 samples during training and present the results in Figure~\ref{fig:experiment2}. We observed that in most tasks the performance using LoRA eventually converges a test accuracy that matches that of full fine-tuning, although the rates of convergence sometimes differ. We list the hyperparameters in Table \ref{tab:hyperparam-acc}

\begin{table}[h]
\centering

\resizebox{0.9\textwidth}{!}{
\begin{tabular}{rcccccc}
\toprule
Task & \textbf{SST-2,QQP,MR,CR} & \textbf{Subj} & \textbf{QNLI} \\
\midrule
 Batch size & 32 & 32 & 24 \\
Learning rate (Full, LoRA fine tuning) & 0.0001 & 0.001 & 0.0005 \\
Trained layer & $W_{q} , W_{v} \ \textrm{(all layers)} $ & $W_{q} , W_{v}$ \ \textrm{(all layers)} &$W_{q} , W_{v} \ \textrm{(all layers)} $ \\
Weight decay  & 0.005 & 0.005 & 0.005\\
\bottomrule
\end{tabular}
}
\caption{Hyperparameters on experiment in Figure \ref{fig:experiment2}}
\label{tab:hyperparam-acc}
\end{table}
 \begin{figure*}[!t]
        \centering
        \begin{subfigure}[b]{0.33\textwidth}
            \centering
            \includegraphics[width=\textwidth]{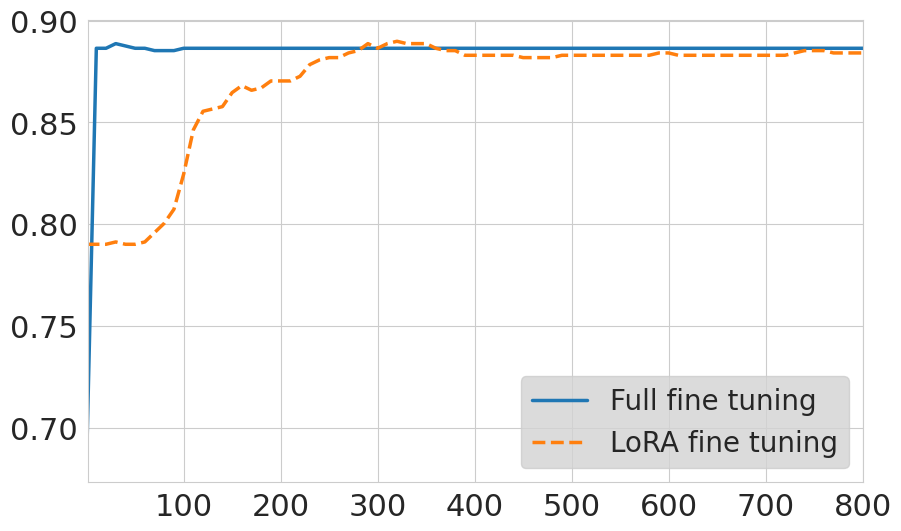}
            \caption{{SST-2}}    
            \label{fig:sste}
        \end{subfigure}
        \begin{subfigure}[b]{0.33\textwidth}  
            \centering 
            \includegraphics[width=\textwidth]{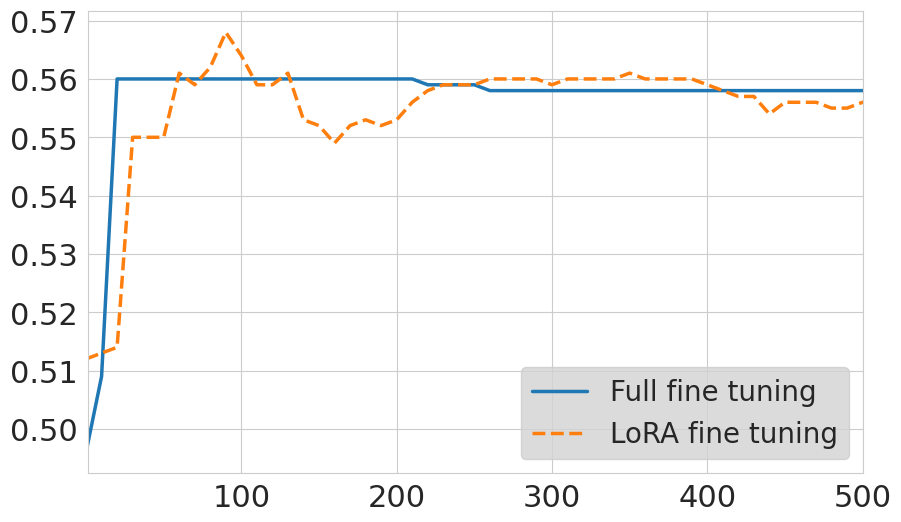}
            \caption
            {{QNLI}}
            \label{fig:qnlie}
        \end{subfigure}
                \begin{subfigure}[b]{0.33\textwidth} 
            \centering 
            \includegraphics[width=\textwidth]{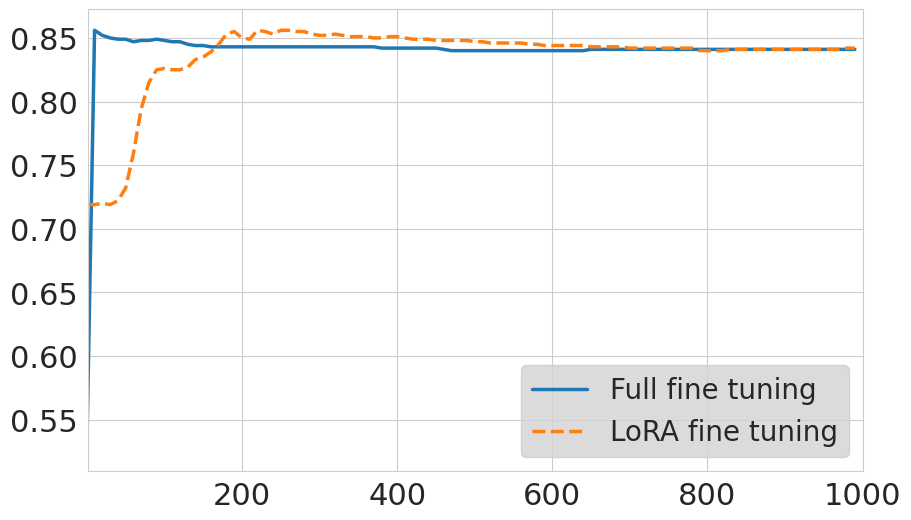}
            \caption
            {{MR}} 
            \label{fig:mre}
        \end{subfigure}
        \vskip\baselineskip
        \begin{subfigure}[b]{0.33\textwidth}   
            \centering 
            \includegraphics[width=\textwidth]{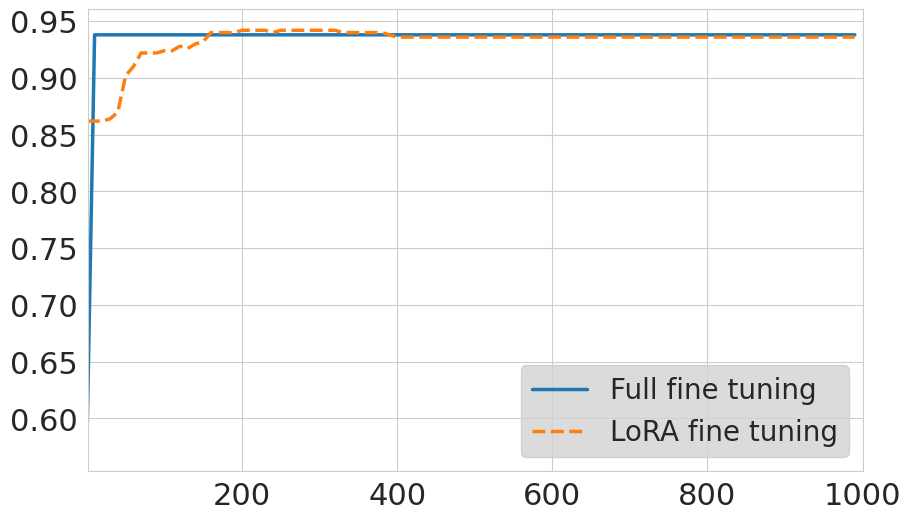}
            \caption
            {{CR }}  
            \label{fig:cre}
        \end{subfigure}
                \begin{subfigure}[b]{0.33\textwidth}  
            \centering 
            \includegraphics[width=\textwidth]{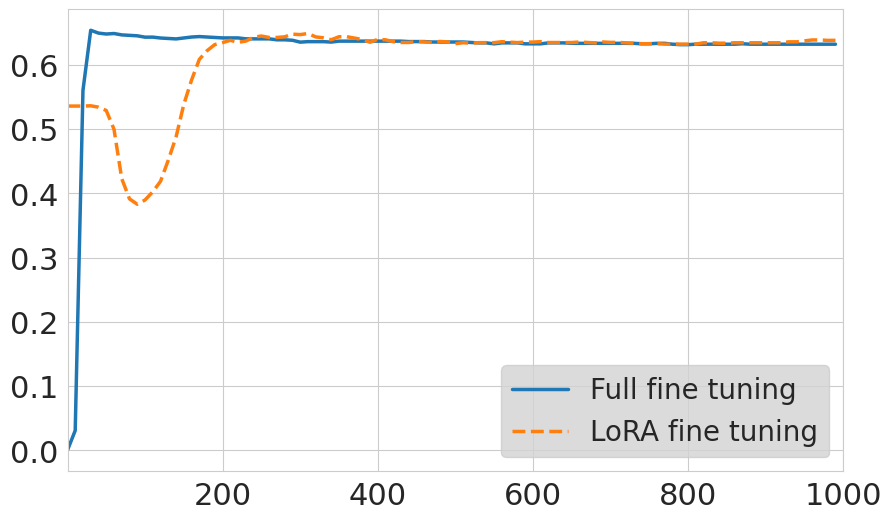}
            \caption
            {{QQP}} 
            \label{fig:qqpe}
        \end{subfigure}
        \begin{subfigure}[b]{0.33\textwidth}   
            \centering 
            \includegraphics[width=\textwidth]{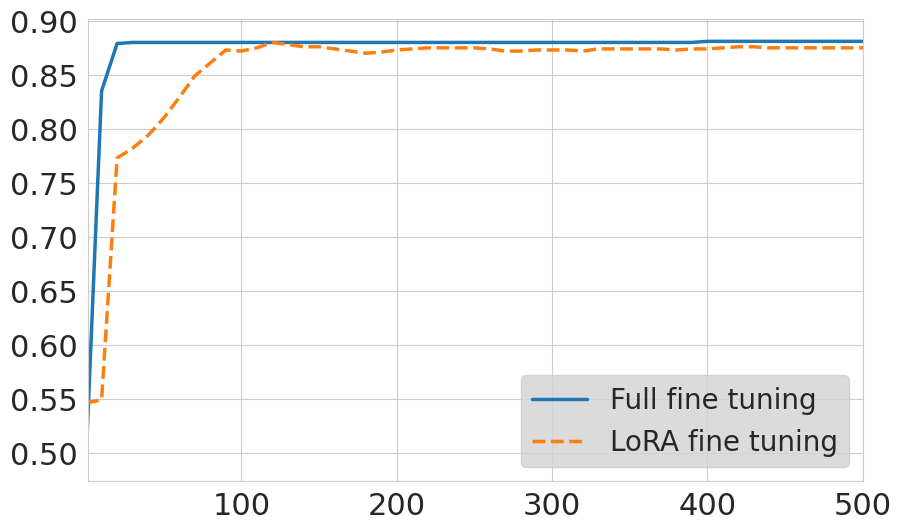}
            \caption
            {{Subj}}  
            \label{fig:subje}
        \end{subfigure}
          \vskip\baselineskip
        \caption{Test curves (accuracy vs. epochs) on different NLP tasks. We used the LoRA rank of 16.} 
        \label{fig:experiment2}
    \end{figure*}
    
For image and speech classification tasks, we also validate the performance of our linearized update to confirm that the accuracy is on par with actual LoRA updates. Accuracies are averaged over 3 runs (See Table~\ref{tab:hyperparam-acc-vision}).

\begin{table}[!h]
\centering

\resizebox{0.7\textwidth}{!}{
\begin{tabular}{rcccccc}
\toprule
Task & \textbf{Image classification} & \textbf{Speech classification} \\
\midrule
Accuracy ( actual / linearized) & 86.20 / 87.00 & 74.67 / 73.67 \\
\bottomrule
\end{tabular}
}
\caption{Accuaricies of LoRA updates on vision and speech classification tasks}
\label{tab:hyperparam-acc-vision}
\end{table}
\end{document}